\documentclass{article}
\usepackage[final,main,nonatbib]{neurips_2026}

\usepackage{caption}
\usepackage[utf8]{inputenc}
\usepackage[T1]{fontenc}
\usepackage{xcolor}
\definecolor{anonurl}{RGB}{120,60,150}
\definecolor{citeblue}{HTML}{1F63A8}
\definecolor{refred}{HTML}{B03A2E}
\usepackage[colorlinks=true,
            linkcolor=refred,
            citecolor=citeblue,
            urlcolor=citeblue]{hyperref}
\usepackage{url}
\usepackage[numbers,sort&compress]{natbib}
\usepackage{booktabs}
\usepackage{multirow}
\usepackage{amsfonts,amsmath,amssymb,amsthm}
\usepackage{mathtools}
\usepackage{nicefrac}
\usepackage{microtype}
\usepackage{algorithm}
\usepackage{algorithmic}
\usepackage{etoc}
\usepackage{tocloft}
\usepackage{cleveref}
\usepackage{wrapfig}
\usepackage{subcaption}
\usepackage{graphicx}
\crefname{assumption}{Assumption}{Assumptions}
\Crefname{assumption}{Assumption}{Assumptions}
\newtheorem{theorem}{Theorem}
\newtheorem{proposition}{Proposition}

\newtheorem{assumption}{Assumption}
\newtheorem{remark}{Remark}

\DeclareMathOperator{\Var}{Var}

\DeclareMathOperator{\Tr}{Tr}

\DeclareMathOperator{\supp}{supp}

\newcommand{\E}{\mathbb{E}}
\newcommand{\R}{\mathbb{R}}
\newcommand{\cD}{\mathcal{D}}
\newcommand{\cT}{\mathcal{T}}
\newcommand{\cS}{\mathcal{S}}
\newcommand{\cA}{\mathcal{A}}
\newcommand{\cL}{\mathcal{L}}
\newcommand{\cN}{\mathcal{N}}

\newcommand{\cI}{\mathcal{I}}
\newcommand{\cH}{\mathcal{H}}
\newcommand{\cM}{\mathcal{M}}

\newcommand{\zz}{\mathbf{z}}
\newcommand{\cZ}{\mathcal{Z}}

\title{Pessimistic Latent Task-aware Optimization \\ for Robust Offline Meta-Reinforcement Learning}

\author{%
  Chang-Hoon Jeong$^{1}$ \quad Kisung Shin$^{2}$ \quad Hanwool Sul$^{2}$ \quad
  Byoung-Tak Zhang$^{1,2}$\thanks{Corresponding author.} \\[0.35em]
  $^{1}$Dept.\ of Computer Science and Engineering, Seoul National University \\
  $^{2}$Interdisciplinary Program in Artificial Intelligence, Seoul National University \\[0.25em]
  \texttt{\{kevinjeong, virtualkss, kangsul, btzhang\}@snu.ac.kr}
}

\begin{document}
\setcounter{footnote}{1}
\maketitle
\setcounter{footnote}{0}

\begin{abstract}
    Context-based offline meta-reinforcement learning (COMRL) has advanced almost entirely through better task representations, which separate task identity from artifacts of the behavior policy so that the latent reflects the task alone. The implicit expectation has been that a clean representation will transfer to out-of-support (parameter-extrapolation) tasks. On such a task, however, the encoder produces a latent outside the training support, where the policy has never been trained. We address this gap with Pessimistic Latent Task-aware Optimization (PLATO), which exposes the policy to off-support latents during training. PLATO assumes a mild manifold hypothesis: task latents lie on a low-dimensional manifold, and out-of-support tasks lie farther from the training centroid. It exploits this geometry by perturbing inferred latents outward from the centroid to synthesize counterfactual tasks. A learned decoder ensemble then rolls out short trajectories at each perturbed latent, and the policy is updated on the imagined transitions under a pessimism weight that down-weights steps on which the ensemble disagrees. We prove that the perturbation leaves the training support by a controllable margin and that the out-of-support generalization gap is bounded by the same perturbation reach, and we give explicit conditions, verified empirically, under which the epistemic disagreement at the perturbed latent stays bounded away from zero. On eight MuJoCo continuous-control benchmarks, PLATO consistently improves out-of-support returns over strong COMRL baselines while remaining competitive on in-distribution tasks. Our code is available at \url{https://github.com/deephoony/plato}.
\end{abstract}

\section{Introduction}
\label{sec:intro}

Building reinforcement learning agents that adapt to a new task from little experience is a long-standing goal of meta-learning~\citep{finn2017maml,rakelly2019efficient,zintgraf2020varibad}. In domains where online interaction is costly or unsafe (robotics, healthcare, autonomous driving), agents must instead learn from a fixed batch of pre-collected experience~\citep{levine2020offline,wu2019behavior}. Offline meta-reinforcement learning (OMRL) combines the two requirements: an agent is trained on a fixed collection of datasets and must adapt to a new task without further interaction~\citep{li2020focal,dorfman2021offline,mitchell2021macaw}. The dominant paradigm is context-based~\citep{rakelly2019efficient,zintgraf2020varibad}: a context encoder $q_\phi(\zz \mid c)$ summarizes a batch of transitions from the new task into a latent code $\zz$, and a shared policy $\pi_\theta(a \mid s, \zz)$ acts on that code. Adaptation thus reduces to inferring $\zz$ from a small context buffer.

Most recent progress in COMRL has come from improving the encoder, with the shared goal of extracting task identity from the context while suppressing the signature of the behavior policy that collected the data. These include contrastive structure on the latent~\citep{li2020focal,yuan2022corro}, information-theoretic objectives bounding the mutual information between $\zz$ and task-relevant transitions~\citep{li2024unicorn}, and max-min mechanisms that directly minimize the mutual information between $\zz$ and the behavior policy~\citep{gao2023csro,nakhaei2025ertrl}. Later work extends these ideas in orthogonal directions~\citep{zhou2024gentle,zhang2025causalcomrl,he2026ctmrl,zhang2025taskshift}. In all of these, the implicit expectation is that a clean representation will, by itself, transfer beyond the training range.

\begin{wrapfigure}{r}{0.5\textwidth}
    \centering
    \vspace{-10pt} 
    \includegraphics[width=0.48\textwidth]{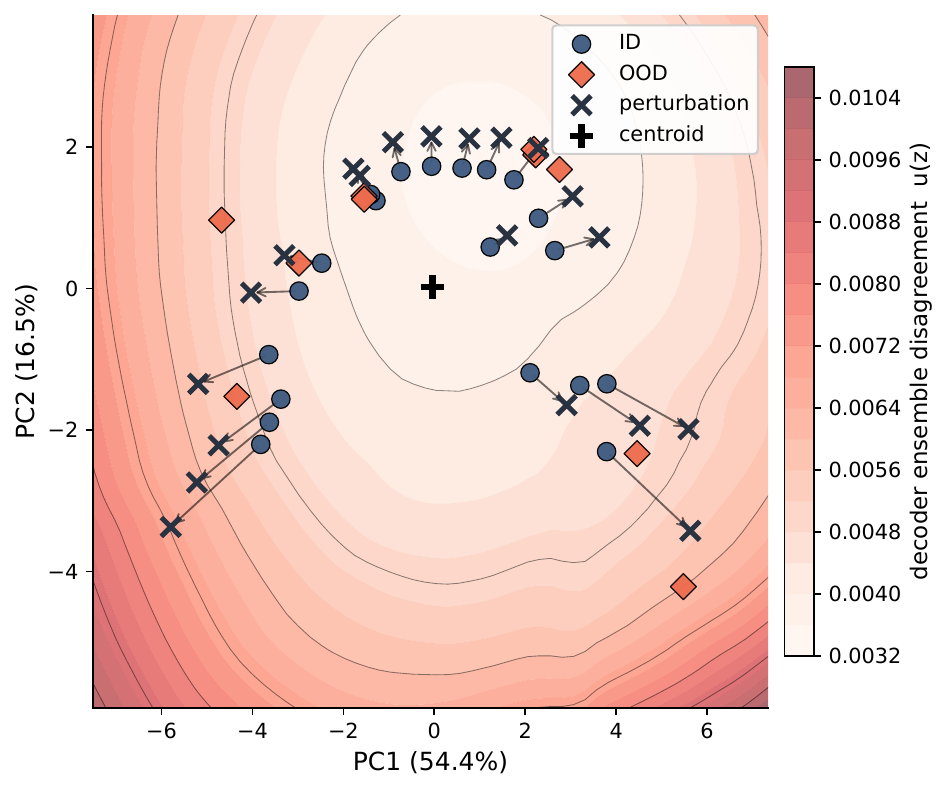}
    \caption{\textbf{Latent geometry on Ant-Goal.} Background: decoder-ensemble disagreement $u(\mathbf{z})$ over the latent space projected by principal component analysis (PCA). In-distribution (ID) tasks (blue) populate the low-$u$ region around the centroid $c$ ($+$). Out-of-support tasks (orange, labeled OOD) exhibit substantially higher disagreement on average, and our perturbations $\mathbf{z} \to \mathbf{z}^{\text{aug}}$ (arrows, $\times$) reach the same high-$u$ outward region, consistent with \Cref{asm:manifold}, \Cref{prop:outward}, and \Cref{thm:nonvanish}.}
    \label{fig:latent_geometry}
    \vspace{-10pt}
\end{wrapfigure}

This expectation overlooks a structural mismatch between the encoder and the policy. During training, the encoder sees a wide range of contexts, including those from tasks at the edge of the training range. The policy does not: it only ever sees latents inferred from training tasks, all of which lie inside the training latent support $\cZ_\text{train} := \supp(q_\phi)$. On an out-of-support (parameter-extrapolation) task, the encoder produces a latent outside $\cZ_\text{train}$, and the policy must act on an input it has never seen. Improving the encoder does not close this gap; a tighter $\cZ_\text{train}$ widens it rather than narrowing it. What the policy needs is training-time exposure to latents outside $\cZ_\text{train}$, which means synthesizing them.

Synthesis requires a direction. We assume a mild manifold hypothesis on the task family: task latents lie on a low-dimensional manifold, and out-of-support tasks lie farther from the training centroid (\Cref{asm:manifold}, \Cref{fig:latent_geometry}). Under this hypothesis, the natural choice is to push inferred latents outward from the centroid, which is the idea behind Pessimistic Latent Task-aware Optimization (PLATO). Each inferred latent $\zz_i$ is perturbed outward to a latent $\zz_i^\text{aug}$ that lies off the training support by a controllable margin (\Cref{sec:lore}, \Cref{prop:outward}). A learned decoder ensemble~\citep{lakshminarayanan2017simple,chua2018pets} then rolls out short trajectories at $\zz_i^\text{aug}$, producing imagined transitions in a counterfactual task just beyond the training range (\Cref{sec:rollout}). Because the rollouts are exploratory by design, the policy is trained on them only under a pessimism weight that decreases with the ensemble disagreement at each step (\Cref{sec:policy}, \Cref{thm:nonvanish}), and by default the critic is trained on real transitions alone (\Cref{thm:decoupling}). \Cref{fig:latent_geometry} illustrates this on Ant-Goal: on average, out-of-support tasks (orange) lie in the outward region of higher disagreement, and our perturbations $\zz \to \zz^\text{aug}$ reach the same region.

We make four contributions. (i) We identify a structural gap in COMRL: representation-focused methods give the policy no training-time exposure to off-support latents, the inputs that out-of-support tasks present to it. (ii) We formalize a manifold hypothesis on the geometry of unseen tasks and use it to design PLATO, combining a centroid-outward latent perturbation with a pessimism-weighted policy update on imagined rollouts. (iii) We analyze the design (\Cref{sec:analysis}): the perturbation reach controls both the off-support margin and a bound on the out-of-support value gap, the off-support disagreement persists under explicit conditions, and imagined data in the critic can incur a bias linear in the model error. (iv) Against five COMRL baselines, PLATO attains the highest average out-of-support return on all eight MuJoCo task families and remains competitive in distribution.

\section{Preliminaries}
\label{sec:background}

We consider a task distribution $p(\cT)$ over Markov decision processes (MDPs) $\cT_i = (\cS, \cA, T_i, R_i, \gamma)$ that share states and actions but differ in transitions or rewards. For each of $N$ training tasks, we are given an offline dataset $\cD_i = \{(s, a, r, s')\}$ collected by a behavior policy $\pi_\beta^i$. At test time, the agent observes a context $c$ from a new task, encodes it as $\zz \sim q_\phi(\zz \mid c)$, and acts with $\pi_\theta(a \mid s, \zz)$. We write $\cZ_\text{train} := \supp(q_\phi)$ for the training latent support and call $\zz \notin \cZ_\text{train}$ off-support; $\zz^\text{aug}$ denotes a latent synthesized by the perturbation of \Cref{sec:lore}. Out-of-support tasks are tasks whose parameters lie beyond the range covered by the training tasks; PLATO targets this kind of extrapolation within a task family, and transfer to qualitatively different task families is outside the scope of this work. As our actor--critic backbone we use BRAC~\citep{wu2019behavior}, which augments the standard actor--critic objective with a dual-form divergence penalty that keeps $\pi_\theta$ close to the behavior distribution.

\paragraph{What kind of uncertainty PLATO needs.}
PLATO trains the policy on imagined transitions generated at off-support latents $\zz^\text{aug}$. These transitions come from a decoder ensemble, a set of $M$ learned dynamics-and-reward models conditioned on the task latent (\Cref{sec:backbone}). The ensemble has never been trained on inputs at $\zz^\text{aug}$, so its predictions there can be unreliable, and the policy must be told how much to trust each imagined step. Predictive uncertainty has two sources~\citep{depeweg2018decomposition,gal2016dropout}: aleatoric uncertainty, the irreducible noise of the environment, and epistemic uncertainty, which arises because the model has not seen enough data. Only the epistemic component carries the signal PLATO needs. Aleatoric noise is the same on training and out-of-support tasks; what changes off-support is how tightly the data constrain the model.

\paragraph{The epistemic component as mutual information.}
We make this precise in the standard way. Let the next state $s'$ be predicted by a decoder with parameters $\theta$, whose posterior given the offline data is $p(\theta \mid \cD)$. The total predictive uncertainty at an input $(s, a, \zz)$ is the marginal entropy $\cH(s' \mid s, a, \zz)$, which averages over both $\theta$ and the noise of $s'$ given $\theta$. The aleatoric part is the uncertainty that remains once $\theta$ is fixed, $\E_{p(\theta \mid \cD)}[\cH(s' \mid s, a, \zz, \theta)]$. Subtracting it from the total isolates the epistemic part, which by the chain rule of entropy is exactly the mutual information (MI) $\cI(s'; \theta \mid s, a, \zz)$:
\begin{equation}
  \cI(s';\, \theta \mid s, a, \zz)
  \;=\; \underbrace{\cH(s' \mid s, a, \zz)}_{\text{total}}
  \;-\; \underbrace{\E_{p(\theta \mid \cD)}\!\bigl[\cH(s' \mid s, a, \zz, \theta)\bigr]}_{\text{aleatoric}}.
  \label{eq:epistemic}
\end{equation}
The mutual information is large exactly when the data fail to pin down $\theta$ at $(s, a, \zz)$. It is therefore small for $\zz$ inside $\cZ_\text{train}$ and large for off-support $\zz$, which is the property PLATO exploits.

\paragraph{Practical estimation.}
Computing \Cref{eq:epistemic} exactly requires the posterior $p(\theta \mid \cD)$, which is intractable. We approximate the posterior with a decoder ensemble~\citep{lakshminarayanan2017simple} of $M$ members trained on bootstrap resamples of the data~\citep{chua2018pets}. Under a Gaussian decoder likelihood, the predictive variance across members is a consistent, monotone proxy for \Cref{eq:epistemic} (proof in \Cref{app:mi_proxy}). PLATO never estimates the mutual information itself; it uses only this variance proxy at $\zz^\text{aug}$, as the uncertainty signal in the pessimism weight (\Cref{sec:policy}).

\section{Method}
\label{sec:method}

\begin{figure}[t]
\centering
\includegraphics[width=\linewidth]{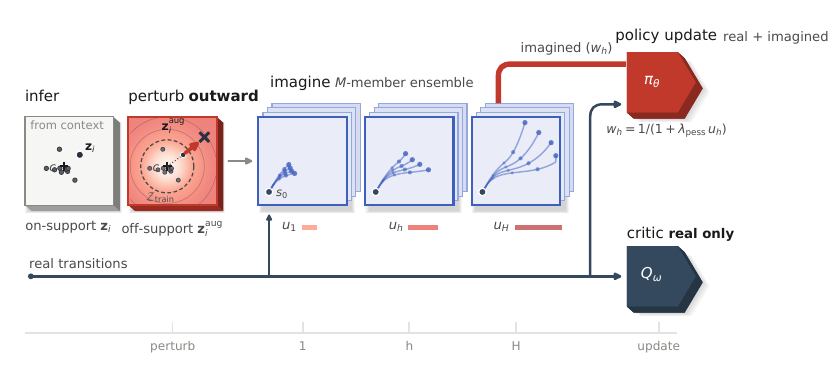}
\caption{\textbf{Conceptual overview of PLATO.} From left to right: a latent $\zz_i$ is inferred from the context of a training task; it is pushed outward from the training centroid $c$ to $\zz_i^\text{aug}$, off the training support $\cZ_\text{train}$ (\Cref{asm:manifold}, \Cref{prop:outward}); and the $M$-member decoder ensemble, conditioned on $\zz_i^\text{aug}$, rolls out a short imagined trajectory from a real state $s_0$, recording the per-step disagreement $u_h$ as its members drift apart. The policy $\pi_\theta$ is updated on real transitions and on the imagined ones, each imagined step weighted by $w_h = 1/(1+\lambda_\text{pess} u_h)$, while the critic $Q_\omega$ is trained on real transitions only (\Cref{thm:decoupling}).}
\label{fig:overview}
\end{figure}

PLATO builds on a representation backbone (\Cref{sec:backbone}) and a manifold hypothesis on the task family (\Cref{sec:hypothesis}), and adds three algorithmic components: a radial outward perturbation that synthesizes counterfactual out-of-support tasks (\Cref{sec:lore}), an imagined rollout of the decoder ensemble at the perturbed latent (\Cref{sec:rollout}), and a pessimism-weighted policy update on the resulting transitions (\Cref{sec:policy}). \Cref{fig:overview} gives a conceptual overview of this training-time pipeline, and \Cref{app:algorithm} gives the full meta-training and meta-testing procedures with pseudocode.

\subsection{Representation backbone}
\label{sec:backbone}
PLATO is agnostic to the choice of representation backbone, but it presupposes a backbone whose latent geometry reflects task identity rather than the behavior policy that collected the data. Since the perturbation in \Cref{sec:lore} uses only batch statistics of the inferred latents and the decoder ensemble takes $\zz$ as an input, PLATO composes with any context encoder; throughout the paper we use the CSRO backbone. Following~\citep{gao2023csro}, we train the context encoder $q_\phi(\zz \mid c)$ with a FOCAL contrastive loss~\citep{li2020focal} to separate task embeddings and with a CLUB upper bound on $\cI(\zz; s, a)$~\citep{cheng2020club} to suppress behavior-policy shortcuts; both objectives, and the BRAC penalty of the actor--critic, are stated in \Cref{app:impl}. In parallel, an ensemble of $M$ bootstrapped decoders $\{f_{\psi_m}\}_{m=1}^M$~\citep{chua2018pets,lakshminarayanan2017simple,an2021edac,ghasemipour2022msg} is fit by squared-error regression on real transitions to predict the normalized next-state delta and reward, $y = \bigl((s' - s)/\sigma_s,\, (r - \mu_r)/\sigma_r\bigr)$, conditioned on $\zz$. The ensemble plays two roles in PLATO: it serves as the rollout model in \Cref{sec:rollout}, and its across-member variance over both output channels is the disagreement signal that drives the pessimism weight in \Cref{sec:policy}. In the default pipeline, the decoded reward enters only this signal, because imagined transitions are used in the actor objective alone; the Q-synth ablation of \Cref{sec:experiments} additionally uses it as the reward of the imagined Bellman target. Architecture details and optimizer settings are in \Cref{app:impl}.

\subsection{Manifold hypothesis on the task family}
\label{sec:hypothesis}

Synthesizing off-support latents requires an assumption about where out-of-support tasks lie relative to the training tasks. We adopt a mild geometric assumption that captures how out-of-support tasks typically arise: they extrapolate beyond a contiguous region of training tasks rather than interpolate between them. Let the task family be parameterized by an intrinsic $\theta \in \Theta \subset \R^k$ with $k$ small. The trained encoder defines a map $\phi : \Theta \to \R^d$, $\theta \mapsto \E[q_\phi(\cdot \mid c_\theta)]$, that takes each task parameter to its mean latent, where $c_\theta$ is a context from task $\theta$.

\begin{assumption}[Manifold hypothesis]
\label{asm:manifold}
The map $\phi$ is a smooth injection, and the latent task manifold $\cM_z := \phi(\Theta)$ has intrinsic dimension $k$. Let $\Theta_\text{train} \subsetneq \Theta$ be the training tasks, $\Theta_\text{out} := \Theta \setminus \Theta_\text{train}$ the out-of-support tasks, and $\cM_\text{train} := \phi(\Theta_\text{train})$ and $\cM_\text{out} := \phi(\Theta_\text{out})$ their latent images. With $c := \E_{\theta \sim \Theta_\text{train}}[\phi(\theta)]$ the training centroid, out-of-support latents lie farther from $c$ on average:
\[
  \E_{\zz \in \cM_\text{out}}[\|\zz - c\|]
  \;\ge\; \E_{\zz \in \cM_\text{train}}[\|\zz - c\|].
\]
\end{assumption}

\Cref{asm:manifold} packages three properties: low intrinsic dimension, smoothness, and outwardness. Low intrinsic dimension is the standard manifold hypothesis~\citep{fefferman2016testing,bengio2013representation} applied to the geometry of tasks rather than data points, reflecting that task families are typically generated by a small number of underlying physical or semantic factors. Smoothness is enforced at training time by the contrastive and mutual-information objectives of \Cref{sec:backbone}, which shape $\phi$ into a Lipschitz embedding of the task parameters. Outwardness reflects how out-of-support tasks typically arise in practice: a policy trained for a band of operating conditions encounters values beyond that band rather than between them. The assumption constrains only the \emph{average} radial position of out-of-support latents; it does not require the manifold to be globally flat or the encoder to be globally Euclidean. It is a sufficient condition for the outward direction to point toward held-out tasks, not a requirement of the method as a whole: the pessimism weight depends on ensemble disagreement rather than on latent geometry, and PLATO retains its advantage on a benchmark where the assumption is not confirmed (\Cref{app:env-pca}). \Cref{fig:latent_geometry} illustrates the geometry on Ant-Goal, and \Cref{app:env-pca} tests the assumption directly.

\subsection{Radial outward perturbation}
\label{sec:lore}

The perturbation rule must satisfy two properties. First, it should push each $\zz_i$ outward from the training centroid, so that the perturbed latent lands in $\cM_\text{out}$ under \Cref{asm:manifold}. Second, it should spend most of the perturbation budget on samples already close to the boundary, since the policy will be tested on $\cM_\text{out}$, not in the interior.

We summarize a meta-batch of $B$ inferred latents by their centroid $c := \tfrac{1}{B}\sum_i \zz_i$, and let $d_i := \|\zz_i - c\|$, $\bar d := \tfrac{1}{B}\sum_i d_i$, $\sigma_d$ the standard deviation of $\{d_i\}$, and $\hat n_i := (\zz_i - c)/d_i$ the outward unit direction. The perturbed latent is
\begin{equation}
  \zz_i^\text{aug} \;=\; \zz_i 
  \;+\; \alpha \,\sigma_d \,
        \underbrace{(d_i / \bar d)^p}_{\text{edge weight}} \,
        \hat n_i.
  \label{eq:lore}
\end{equation}
The unit direction $\hat n_i$ points away from the centroid, so the perturbation only increases the centroid distance, to $\|\zz_i^\text{aug} - c\| = d_i + \alpha\sigma_d(d_i/\bar d)^p$. The edge weight $(d_i / \bar d)^p$ scales the perturbation with how peripheral the sample already is: with the default $p = 2$, a latent at twice the mean distance moves four times as far as one at the mean distance, so interior samples receive small perturbations and boundary samples large ones. The local scale $\sigma_d$ normalizes the perturbation to the spread of the current batch, so $\alpha$ has a consistent meaning across encoders. The construction uses batch statistics alone and requires no task labels.

The two hyperparameters $\alpha > 0$ and $p \ge 1$ control the perturbation reach and the concentration on boundary samples, respectively. \Cref{prop:outward} gives the threshold at which $\zz_i^\text{aug}$ leaves $\cZ_\text{train}$, and \Cref{rem:lore_bound} shows that, once the perturbation crosses that threshold, the off-support disagreement at $\zz_i^\text{aug}$ stays positive wherever the margin condition of \Cref{thm:nonvanish} holds.

\subsection{Imagined rollout and the disagreement signal}
\label{sec:rollout}

We treat the decoder ensemble as a task-conditioned rollout model. Starting from a real state $s_0$ of task $i$ and holding the latent fixed at $\zz^\text{aug} = \zz_i^\text{aug}$, we unroll a short trajectory of length $H$: for $h = 0, \ldots, H-1$,
\begin{equation}
  \begin{gathered}
  a_h \sim \pi_\theta(\cdot \mid s_h, \zz^\text{aug}), \qquad
  s_{h+1} \;=\; s_h + \sigma_s \odot \bar f^{(\Delta s)}(s_h, a_h, \zz^\text{aug}), \\
  \hat r_h \;=\; \mu_r + \sigma_r\, \bar f^{(r)}(s_h, a_h, \zz^\text{aug}),
  \end{gathered}
  \label{eq:rollout}
\end{equation}
where $\bar f = \tfrac{1}{M}\sum_m f_{\psi_m}$ is the ensemble mean, $\bar f^{(\Delta s)}$ and $\bar f^{(r)}$ are its state-delta and reward channels, and $\sigma_s$, $\mu_r$, and $\sigma_r$ are the normalization constants of $y$, computed once from the offline data. At each step we also record the across-member variance,
\begin{equation}
  u_h \;=\; \tfrac{1}{d_s + 1}\sum_{j} \Var_m\bigl[f_{\psi_m}^{(j)}(s_h, a_h, \zz^\text{aug})\bigr],
  \label{eq:u}
\end{equation}
where $d_s = \dim(\cS)$ and the sum runs over the $d_s$ state-delta channels and the reward channel, so reward uncertainty is part of the signal. By \Cref{prop:variance_mi} (proved in \Cref{app:mi_proxy}), $u_h$ is a monotone proxy for the epistemic mutual information at $(s_h, a_h, \zz^\text{aug})$, and \Cref{thm:nonvanish} gives conditions under which it stays bounded away from zero off the training support, so that it remains informative where PLATO operates.

\subsection{Pessimism-weighted policy improvement}
\label{sec:policy}

The imagined rollouts at $\zz^\text{aug}$ leave the support of the offline data: they generate transitions in counterfactual tasks the agent has never seen. PLATO therefore uses them pessimistically. The policy is updated on each imagined transition only in proportion to the model's confidence at that step, weighted by
\begin{equation}
  w_h \;=\; \frac{1}{1 + \lambda_\text{pess}\,u_h},
  \label{eq:weight}
\end{equation}
which discounts each step by a monotone transformation of the MI proxy; the hyperparameter $\lambda_\text{pess}$ controls the strength of the pessimism. The weight stays close to one where the members agree and decays as $1/(\lambda_\text{pess} u_h)$ where they disagree strongly. The actor loss combines a real-data term and a weighted synthetic-rollout term, $\cL_\pi(\theta) = \cL_\pi^\text{real}(\theta) + \beta_\text{synth}\,\cL_\pi^\text{synth}(\theta)$, with
\begin{align}
  \cL_\pi^\text{real}(\theta) &= \E_{(s, a) \sim \cD}\!\left[\eta\log\pi_\theta(a' \mid s, \zz) - Q_\omega(s, a', \zz) + \alpha_\text{BRAC}\, D_\text{BRAC}(s, a', a; \zz)\right], \label{eq:actor_real} \\
\cL_\pi^\text{synth}(\theta) &= \frac{1}{H}\sum_{h=0}^{H-1}\E_{s_h}\!\left[w_h\bigl(\eta\,\log\pi_\theta(a'_h \mid s_h, \zz^\text{aug}) - Q_\omega(s_h, a'_h, \zz^\text{aug})\bigr)\right], \label{eq:actor_synth}
\end{align}
where $a' \sim \pi_\theta(\cdot \mid s, \zz)$, $a'_h \sim \pi_\theta(\cdot \mid s_h, \zz^\text{aug})$, $Q_\omega$ is the minimum of the twin critics, and $\eta$ is the entropy temperature. By default, the critic is trained only on real transitions,
\begin{equation}
  \cL_Q(\omega_k) \;=\; \E_{(s, a, r, s') \sim \cD}\bigl[\bigl(Q_{\omega_k}(s, a, \zz) - r - \gamma\,\bar V(s', \zz)\bigr)^2\bigr], \qquad k \in \{1, 2\},
  \label{eq:critic}
\end{equation}
where $\bar V$ is a slowly updated target value network (reward scaling and terminal masking omitted), so imagined transitions never enter a Bellman target; \Cref{thm:decoupling} quantifies the bias of relaxing this default, which we examine empirically in \Cref{sec:experiments}.

\section{Analysis}
\label{sec:analysis}

We establish four properties of PLATO. The perturbation rule in \Cref{eq:lore} pushes latents off the training support by a controllable margin (\Cref{prop:outward}); under conditions we make explicit, the resulting epistemic disagreement stays bounded away from zero (\Cref{thm:nonvanish}), so the pessimism weight in \Cref{eq:weight} stays informative; mixing imagined data into the critic introduces a bias linear in the rollout-model error (\Cref{thm:decoupling}), motivating the real-only default; and the out-of-support generalization gap is bounded by the perturbation reach (\Cref{prop:gengap}). Full proofs are in \Cref{app:proofs}.

\subsection{The perturbation pushes latents off the training support}
\label{sec:analysis_outward}

\begin{proposition}[Off-support guarantee]
\label{prop:outward}
Let $\zz_i \in \cZ_\text{train}$ with $d_i = \|\zz_i - c\| > 0$, and let $\zz_i^\text{aug}$ be defined by \Cref{eq:lore} with parameters $(\alpha, p)$ and batch statistics $(c, \bar d, \sigma_d)$. Then $\|\zz_i^\text{aug} - c\| = d_i + \alpha\,\sigma_d\,(d_i/\bar d)^p$. If $\cZ_\text{train} \subseteq B(c, R)$ for some $R > 0$, then $\zz_i^\text{aug} \notin \cZ_\text{train}$ whenever $\alpha\,\sigma_d\,(d_i/\bar d)^p > R - d_i$. For any peripheral $\zz_i$ with $d_i \ge \bar d$, the choice $\alpha \ge (R - d_i)/\sigma_d$ suffices.
\end{proposition}

The proof is purely geometric (\Cref{app:proof_outward}); combined with \Cref{asm:manifold}, the proposition places $\zz_i^\text{aug}$ in the out-of-support region in expectation.

\subsection{Persistence of off-support disagreement}
\label{sec:analysis_nonvanish}

We now give conditions for the property suggested by \Cref{eq:epistemic}: off the training support, the epistemic disagreement stays bounded below for the remainder of training. Let $\bar\Sigma(\zz) := \Var_m[f_{\psi_m}(s, a, \zz)]$ be the across-ensemble predictive covariance at fixed $(s, a)$. Under the Gaussian decoder likelihood,
\[
  \cI(s';\theta \mid s, a, \zz) \;\le\; \tfrac{1}{2}\log\det\!\bigl(I + \sigma^{-2}\bar\Sigma(\zz)\bigr),
\]
which is monotone in $\bar\Sigma$ (\Cref{prop:variance_mi}).

\begin{assumption}[Parameter-Lipschitz decoders]
\label{asm:lipvar}
Each decoder $f_{\psi_m}$ is $L_\psi$-Lipschitz in its parameters $\psi_m$, uniformly over the compact set of inputs at which rollouts are generated.
\end{assumption}

\begin{assumption}[Transport budget with margin]
\label{asm:margin}
Fix an input $(s, a, \zz)$ and an anchor step $t_0$, and write $\bar\Sigma_t$ for the ensemble covariance at training step $t$. Let $D_m := \sum_{t \ge t_0}\|\psi_m^{t+1} - \psi_m^{t}\|$ be the parameter movement of member $m$ after $t_0$ and $\bar D := \bigl(\tfrac{1}{M}\sum_m D_m^2\bigr)^{1/2}$. Then $\sqrt{\Tr\bar\Sigma_{t_0}(\zz)} > L_\psi \bar D$.
\end{assumption}

\begin{theorem}[Persistence of off-support disagreement]
\label{thm:nonvanish}
Under \Cref{asm:lipvar,asm:margin}, the ensemble disagreement at $\zz$ stays bounded away from zero for the rest of training,
\[
  \inf_{t \ge t_0} \Tr\bar\Sigma_t(\zz) \;\ge\; c \;:=\; \bigl(\sqrt{\Tr\bar\Sigma_{t_0}(\zz)} - L_\psi \bar D\bigr)^2 \;>\; 0,
\]
and consequently $\cI(s';\theta \mid s, a, \zz) > 0$ for all $t \ge t_0$ under the ensemble's empirical posterior.
\end{theorem}

The proof (\Cref{app:proof_nonvanish}) is short: the disagreement is the distance of the stacked member predictions from their consensus, and it can shrink by at most the Lipschitz-weighted parameter movement $L_\psi \bar D$. The substance of the theorem therefore lies in \Cref{asm:margin}, which separates on- from off-support inputs. At an on-support latent, training spends its parameter movement on reducing the disagreement at that very input, so the margin fails and the disagreement can vanish, as it should. At an off-support latent such as $\zz^\text{aug}$, which \Cref{prop:outward} places outside $\cZ_\text{train}$, the movement is not directed at that input, and the realized reduction falls well short of the budget. Because $\bar D$ is a tail sum that shrinks as training converges, the margin is easier to satisfy at later anchors, so a condition checked at a single step yields a guarantee for every later step. \Cref{thm:nonvanish} is thus a conditional guarantee, and we verify its conclusion directly on training checkpoints (\Cref{app:proof_nonvanish}).

\begin{remark}[Reach-disagreement coupling]
\label{rem:lore_bound}
For a peripheral training latent ($d_i \ge \bar d$), \Cref{prop:outward} shows that $\zz_i^\text{aug}$ leaves $\cZ_\text{train}$ once $\alpha\sigma_d (d_i/\bar d)^p > R - d_i$, with a displacement of $\alpha\sigma_d(d_i/\bar d)^p$. Where \Cref{asm:margin} holds at $\zz_i^\text{aug}$, \Cref{thm:nonvanish} then keeps the disagreement signal there bounded away from zero. A single hyperparameter $\alpha$ thus sets both how far the perturbation reaches and where the pessimism signal is evaluated.
\end{remark}

\subsection{Bias from imagined data in the critic}
\label{sec:analysis_decoupling}

We now quantify the trade-off behind PLATO's default of training the critic on real transitions only. The following result concerns an actor--critic whose critic is trained on a mixture of real and imagined data.

\begin{theorem}[Critic bias under mixed-data Bellman targets]
\label{thm:decoupling}
Consider an offline actor--critic where the critic is trained on $\rho_\beta = (1 - \beta)\rho_\text{real} + \beta\rho_\text{imag}$ for some $\beta \in [0, 1]$, with $\rho_\text{imag}$ generated by a rollout model $\hat P$. Let $\hat Q_\beta$ denote the Bellman fixed point under $\rho_\beta$ and $Q^*$ the optimal Q-function under the true MDP. For any $\varepsilon_P > 0$, there exist transition kernels $(P, \hat P)$ with total variation (TV) distance $\sup_{s,a}\|\hat P(\cdot\mid s,a) - P(\cdot\mid s,a)\|_\text{TV} \le \varepsilon_P$ such that
\begin{equation}
  \|\hat Q_\beta - Q^*\|_\infty \;\ge\; \frac{\gamma\beta}{1 - \gamma}\,\varepsilon_P - O(\varepsilon_\text{data}),
  \label{eq:contamination}
\end{equation}
where $\varepsilon_\text{data}$ is the standard offline-RL estimation error on the real-data share. In particular, when $\beta = 0$, $\|\hat Q_0 - Q^*\|_\infty$ has no term linear in $\varepsilon_P$.
\end{theorem}

\Cref{thm:decoupling} is a worst-case bound: it identifies pairs $(P, \hat P)$ for which mixing imagined data into the Bellman target induces a bias linear in $\varepsilon_P$. When $\varepsilon_P$ is small, even the worst-case bias is small, and designs in the style of MOPO~\citep{yu2020mopo} and MOReL~\citep{kidambi2020morel} that mix imagined transitions into the Bellman target can be effective. PLATO, however, operates at $\zz^\text{aug}$, a latent on which the decoder ensemble was never trained, so $\varepsilon_P$ there can be large and the bound becomes informative. This motivates holding the critic to real transitions by default, while leaving the alternative as an empirical knob whose trade-off we examine in \Cref{sec:experiments} (proof in \Cref{app:proof_decoupling}).

\subsection{Out-of-support generalization gap is bounded by the perturbation reach}
\label{sec:analysis_gengap}

The remaining question is whether off-support \emph{exposure} translates into off-support \emph{generalization}, measured in value.

\begin{proposition}[Out-of-support generalization gap is bounded by the perturbation reach]
\label{prop:gengap}
Let $J(\pi; \theta) := \E_{\pi(\cdot \mid \cdot, \phi(\theta))}[\sum_t \gamma^t r_t]$, and suppose $\theta \mapsto J(\pi; \theta)$ is $L_J$-Lipschitz in $\zz = \phi(\theta)$ along $\cM_z$. For any test task $\theta_\star \in \Theta_\text{out}$ and any training task $\theta_0 \in \Theta_\text{train}$,
\[
  |J(\pi; \theta_\star) - J(\pi; \theta_0)| \;\le\; L_J \cdot d_{\cM}\!\bigl(\phi(\theta_\star),\, \phi(\theta_0)\bigr).
\]
If $\phi(\theta_\star)$ lies within the perturbation reach $\{\zz_i^\text{aug} : \zz_i \in \cM_\text{train}\}$, choosing $\theta_0$ to minimize $\|\phi(\theta_\star) - \zz_i^\text{aug}\|$ gives, modulo encoder-smoothness terms,
\[
  |J(\pi; \theta_\star) - J(\pi; \theta_0)| \;\le\; L_J \cdot \alpha\,\sigma_d\,(d_\text{max}/\bar d)^p.
\]
\end{proposition}

The proof (\Cref{app:proof_gengap}) is a Lipschitz argument along the manifold combined with the geometric bound of \Cref{prop:outward}, in the spirit of domain adaptation bounds~\citep{bendavid2010theory,mansour2009domain}. The out-of-support gap is controlled by the perturbation reach $\alpha\sigma_d(d/\bar d)^p$, not by the ambient dimension. By \Cref{prop:outward,prop:gengap}, increasing $\alpha$ simultaneously enlarges (i) the reach of the perturbation and (ii) the value-level radius around training tasks within which PLATO offers a non-trivial guarantee, both by theorem. By \Cref{rem:lore_bound}, it also (iii) keeps the off-support disagreement that determines $w_h$ positive wherever the margin condition holds. The same single hyperparameter governs all three.

\section{Related Work}
\label{sec:related}

\paragraph{Context-based offline meta-RL.}
A productive line of COMRL work invests in the task encoder, separating task identity from the signature of the behavior policy that collected the data. PEARL~\citep{rakelly2019efficient} and VariBAD~\citep{zintgraf2020varibad} introduced the context-based paradigm in online meta-RL (Offline PEARL, the offline variant of PEARL, is a standard COMRL baseline), with subsequent offline work adding contrastive structure~\citep{li2020focal,yuan2022corro}, suppressing behavior-policy shortcuts~\citep{gao2023csro,nakhaei2025ertrl}, and unifying these objectives within an information-theoretic framework~\citep{li2024unicorn}. Recent extensions address data limitations~\citep{zhou2024gentle}, causal factorization~\citep{zhang2025causalcomrl}, cross-task contexts~\citep{he2026ctmrl}, and representation shift during policy improvement~\citep{zhang2025taskshift}. A separate thread replaces the encoder--policy decomposition with a sequence model~\citep{wang2024metadt,ni2023metadiffuser}. PLATO is orthogonal: it is encoder-agnostic and intervenes at the policy level, exposing the policy to off-support latents. Closest to our approach, UNICORN explores, as an extension of its framework, perturbing an inferred task representation with Gaussian noise and training the agent on decoded rollouts~\citep[Sec.~4.2]{li2024unicorn}; PLATO develops this idea into a method with a structured outward perturbation, per-step pessimism, and a critic held to real data. \Cref{app:related} gives a detailed comparison with each baseline.

\paragraph{Model-based offline RL and uncertainty penalties.}
PLATO is closest in spirit to model-based offline RL, where a learned dynamics model produces imagined transitions and an uncertainty estimate quantifies the cost of trusting them. MOPO~\citep{yu2020mopo} and MOReL~\citep{kidambi2020morel} penalize the reward with an ensemble-disagreement term, and MerPO~\citep{lin2022merpo} extends this idea to offline meta-RL by mixing imagined transitions into the Bellman target. PLATO differs in two ways: imagined data enter only the policy gradient, avoiding the bias of \Cref{thm:decoupling}, and rollouts are generated at perturbed \emph{latents} rather than \emph{states}; \Cref{app:merpo} compares the two methods empirically.

\paragraph{Other directions.}
Manifold Mixup~\citep{verma2019manifold} perturbs latent representations for supervised generalization; PLATO adapts this idea to offline meta-RL with a structured outward perturbation. Another thread addresses the shift between training and adaptation contexts at adaptation time, through in-distribution episode selection~\citep{wang2023idaq}, flow-based correction~\citep{wang2026flora}, or reward-free online rollouts~\citep{pong2022selfsup}; PLATO instead intervenes at training time.

\section{Experiments}
\label{sec:experiments}

\begin{table}[t]
\caption{\textbf{Out-of-support returns on the held-out task splits.} Mean $\pm$ standard deviation over 5 seeds at the last epoch. Boldface marks the highest average return per row.}
\label{tab:ood}
\centering\small\setlength{\tabcolsep}{6pt}
\begin{tabular}{lcccccc}
\toprule
Environment & Offline PEARL & FOCAL & CSRO & UNICORN & ER-TRL & PLATO \\
\midrule
Ant-Dir          & -31$\pm$33 & 68$\pm$58 & -17$\pm$125 & -49$\pm$67 & 67$\pm$48 & \textbf{139$\pm$25} \\
Ant-Goal         & -491$\pm$58 & -480$\pm$30 & -527$\pm$48 & -517$\pm$63 & -489$\pm$43 & \textbf{-472$\pm$20} \\
Cheetah-Vel      & -230$\pm$12 & -199$\pm$30 & -188$\pm$14 & -203$\pm$11 & -201$\pm$11 & \textbf{-179$\pm$11} \\
Humanoid-Dir     & 687$\pm$58 & 499$\pm$75 & 624$\pm$83 & 580$\pm$79 & 457$\pm$45 & \textbf{806$\pm$97} \\
\midrule
Hopper-Mass      & 402$\pm$76 & 467$\pm$43 & 494$\pm$59 & 397$\pm$55 & 444$\pm$62 & \textbf{508$\pm$26} \\
Hopper-Friction  & 366$\pm$31 & 334$\pm$55 & 358$\pm$48 & 338$\pm$64 & 357$\pm$21 & \textbf{369$\pm$29} \\
Walker-Mass      & 321$\pm$49 & 340$\pm$77 & 315$\pm$66 & 283$\pm$122 & 366$\pm$27 & \textbf{401$\pm$54} \\
Walker-Friction  & 347$\pm$113 & 287$\pm$77 & 361$\pm$80 & 451$\pm$31 & 454$\pm$74 & \textbf{649$\pm$27} \\
\bottomrule
\end{tabular}
\end{table}

We evaluate PLATO on eight MuJoCo offline meta-RL task families, four with reward shift (Ant-Goal, Ant-Dir, Cheetah-Vel, Humanoid-Dir) and four with dynamics shift (Hopper-Mass, Hopper-Friction, Walker-Mass, Walker-Friction). For each family, we hold out a portion of the task parameter range from training and use it as the out-of-support evaluation split. We compare against five COMRL baselines under the same datasets and protocol: Offline PEARL~\citep{rakelly2019efficient}, FOCAL~\citep{li2020focal}, CSRO~\citep{gao2023csro}, UNICORN~\citep{li2024unicorn}, and ER-TRL~\citep{nakhaei2025ertrl}. We report the normalized return (mean $\pm$ standard deviation over 5 seeds) separately on the in-distribution (ID) and out-of-support splits. Per-environment task ranges and held-out splits are in \Cref{app:env}.

We organize the experiments around four questions. \textbf{(Q1)} Does PLATO improve out-of-support returns? \textbf{(Q2)} Does it preserve in-distribution performance? \textbf{(Q3)} Which components drive the out-of-support gain? \textbf{(Q4)} Does the off-support disagreement signal behave as predicted by \Cref{thm:nonvanish}?

\paragraph{Out-of-support performance (Q1).}
\Cref{tab:ood} reports out-of-support returns on all eight environments. PLATO attains the highest average return on every benchmark; the margin is modest on the Hopper families, Ant-Goal, and Cheetah-Vel, larger on Ant-Dir and Walker-Mass, and substantial on Humanoid-Dir and Walker-Friction. PLATO also has the lowest seed variance of all methods on Ant-Dir, Ant-Goal, Hopper-Mass, and Walker-Friction.

\paragraph{In-distribution performance (Q2).}
\Cref{tab:id} shows that the out-of-support gains do not sacrifice in-distribution competitiveness. The Hopper-Friction gap reflects a small trade-off: the synthetic gradient that drives the out-of-support gains also slightly redirects actor updates on ID inputs, and this cost is environment-dependent. Otherwise, the design behaves as predicted: the perturbation acts only on synthesized off-support latents, the pessimism weight $w_h$ stays close to one on in-support inputs (\Cref{fig:disagreement}), and the critic is trained only on real data.

\begin{table}[t]
\caption{\textbf{In-distribution performance on MuJoCo benchmarks.} Same format as \Cref{tab:ood}. Per-environment task ranges and ID/out-of-support split definitions are in \Cref{app:env}.}
\label{tab:id}
\centering\small\setlength{\tabcolsep}{6pt}
\begin{tabular}{lcccccc}
\toprule
Environment & Offline PEARL & FOCAL & CSRO & UNICORN & ER-TRL & PLATO \\
\midrule
Ant-Dir          & 183$\pm$39 & 589$\pm$20 & 547$\pm$49 & 321$\pm$45 & 605$\pm$34 & \textbf{617$\pm$29} \\
Ant-Goal         & -242$\pm$15 & -218$\pm$10 & -214$\pm$13 & -249$\pm$12 & \textbf{-207$\pm$11} & -215$\pm$10 \\
Cheetah-Vel      & -73$\pm$16 & -69$\pm$16 & -68$\pm$7 & -64$\pm$4 & -64$\pm$6 & \textbf{-61$\pm$11} \\
Humanoid-Dir     & 688$\pm$74 & 550$\pm$44 & 632$\pm$47 & 616$\pm$46 & 529$\pm$61 & \textbf{858$\pm$43} \\
\midrule
Hopper-Mass      & 524$\pm$61 & 503$\pm$59 & 556$\pm$31 & 517$\pm$53 & 502$\pm$53 & \textbf{584$\pm$12} \\
Hopper-Friction  & 517$\pm$46 & 526$\pm$68 & \textbf{584$\pm$15} & 474$\pm$101 & 521$\pm$44 & 526$\pm$42 \\
Walker-Mass      & 298$\pm$62 & 270$\pm$44 & 311$\pm$52 & 315$\pm$109 & 373$\pm$42 & \textbf{507$\pm$53} \\
Walker-Friction  & 416$\pm$133 & 333$\pm$87 & 386$\pm$111 & 513$\pm$65 & 493$\pm$104 & \textbf{682$\pm$18} \\
\bottomrule
\end{tabular}
\end{table}

\paragraph{Component ablations (Q3).}
\begin{wrapfigure}{r}{0.5\textwidth}
\centering
\vspace{-1.2em}
\includegraphics[width=0.48\textwidth]{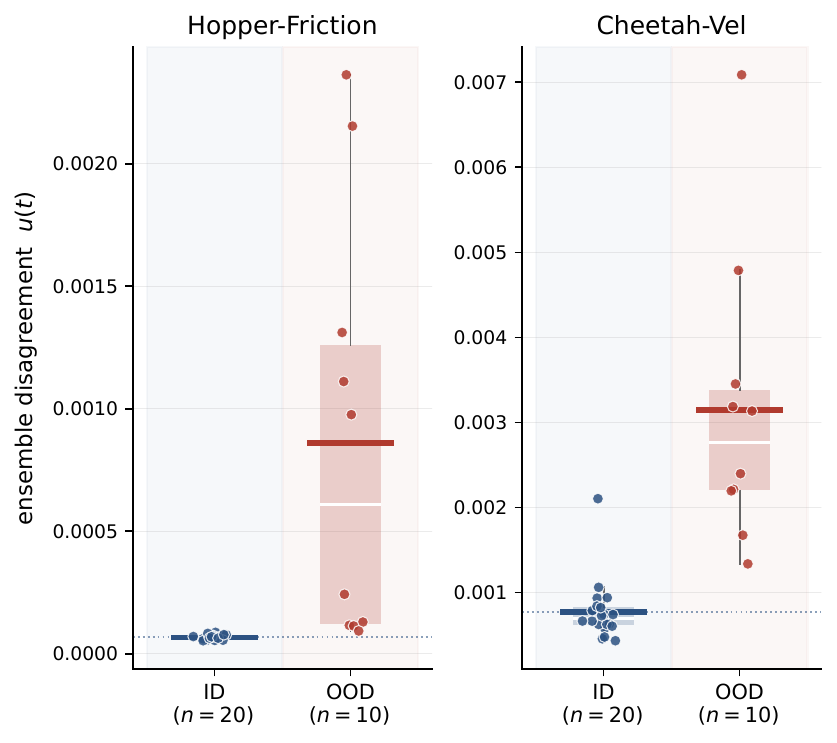}
\vspace{-1em}
\caption{\textbf{Off-support disagreement stays large.} Per-task ensemble disagreement at the trained PLATO checkpoint. ID values (blue) are near zero, whereas out-of-support values (red, labeled OOD) are several-fold to an order of magnitude larger, consistent with \Cref{thm:nonvanish}.}
\label{fig:disagreement}
\vspace{-1.0em}
\end{wrapfigure}
\Cref{fig:ablation_ood} reports out-of-support learning curves for the full algorithm and three ablations: \textbf{no-direction} (random unit vector in place of $\hat n_i$), \textbf{no-weight} ($\lambda_\text{pess} = 0$), and \textbf{Q-synth} (critic also trained on imagined Bellman targets $\hat r_h + \gamma \bar V(s_{h+1}, \zz^\text{aug})$ with the decoded reward of \Cref{eq:rollout}, the variant analyzed by \Cref{thm:decoupling}). The \textbf{no-direction} variant exceeds the default on Humanoid-Dir but underperforms it on the remaining five: across task families, the centroid-outward choice is more consistent than an arbitrary off-support direction. The \textbf{no-weight} and \textbf{Q-synth} variants underperform the default on all six environments, with the largest gaps on the dynamics-shift families. This is consistent with \Cref{thm:decoupling}: $\varepsilon_P$ at $\zz^\text{aug}$ is larger where transitions differ across tasks, so unweighted rollouts in the actor (\textbf{no-weight}) and rollouts in the Bellman target (\textbf{Q-synth}) both cost more there. We retain the critic-real default.

\begin{figure}[t]
\centering
\includegraphics[width=\linewidth]{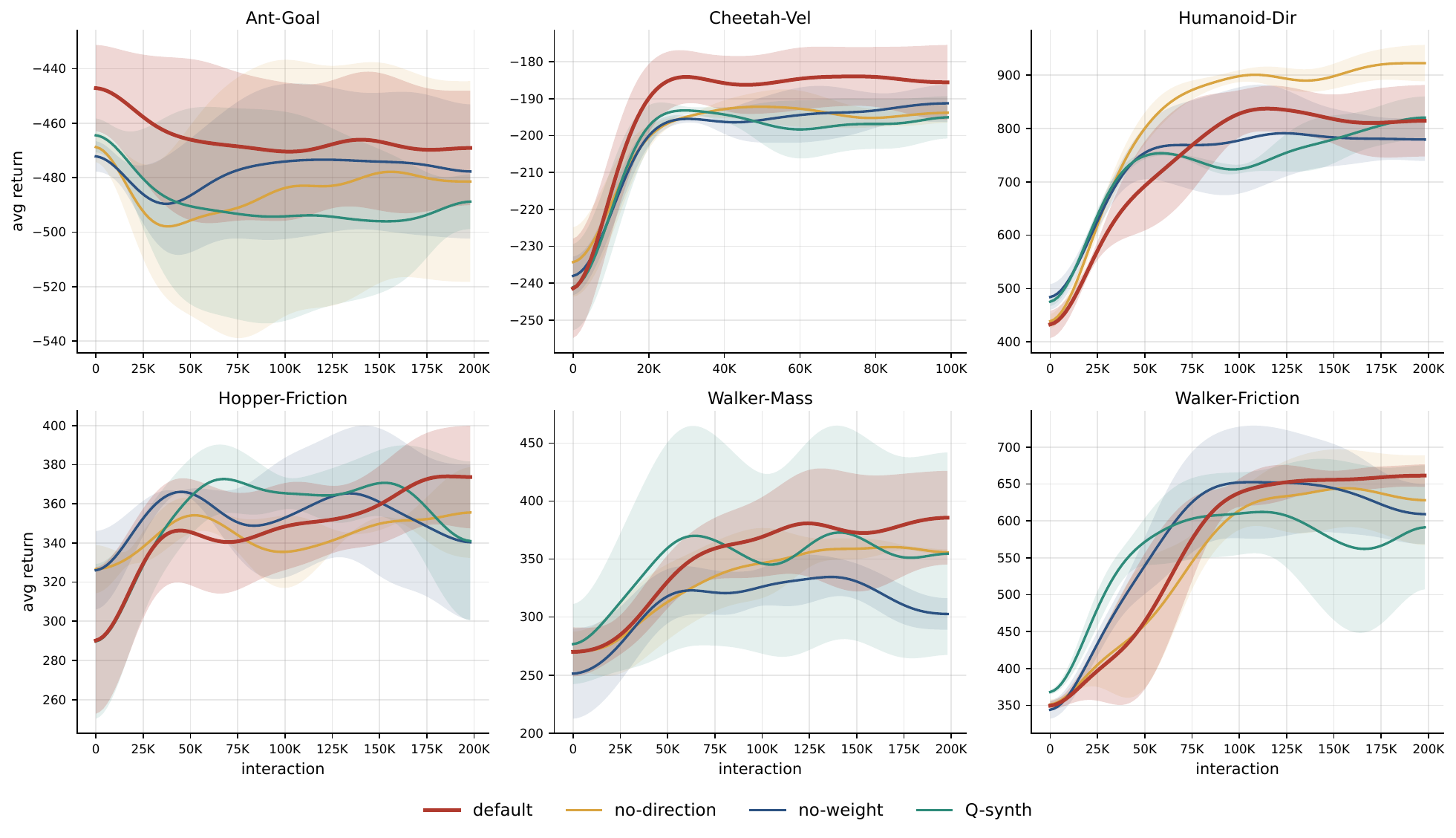}
\vspace{-0.8em}
\caption{\textbf{Component ablations of PLATO on out-of-support returns.} Learning curves on six MuJoCo environments. \textbf{default} is the full algorithm; \textbf{no-direction} replaces the centroid-outward direction $\hat n_i$ with a random unit vector; \textbf{no-weight} sets $\lambda_\text{pess} = 0$; \textbf{Q-synth} also trains the critic on imagined Bellman targets (the variant analyzed by \Cref{thm:decoupling}). Curves show mean over 5 seeds; shaded regions indicate one standard deviation.}
\label{fig:ablation_ood}
\vspace{-1em}
\end{figure}

\paragraph{The off-support disagreement signal (Q4).}
\Cref{thm:nonvanish} gives conditions under which the disagreement $u(\zz)$ of the decoder ensemble stays bounded away from zero on off-support inputs, whereas on the training support it is expected to decay as the decoders fit the data. \Cref{fig:disagreement} measures this directly: we take the trained PLATO checkpoint, encode each task's offline buffer into a latent, and report the per-task disagreement $u(t) = \tfrac{1}{d_s + 1}\sum_j \Var_m[f_{\psi_m}^{(j)}]$. On Hopper-Friction, the disagreement averages $6.7 \times 10^{-5}$ on ID tasks and $8.8 \times 10^{-4}$ on out-of-support tasks, a thirteen-fold gap; on Cheetah-Vel, the gap is fourfold ($7.7 \times 10^{-4}$ vs.\ $3.2 \times 10^{-3}$). Tracking the same Hopper-Friction tasks across training checkpoints, on-support disagreement decays by a factor of $7$--$8$, while off-support disagreement never falls below $11.5$ times the final on-support level and ends $27.8$ times higher (\Cref{app:proof_nonvanish}); the conclusion of \Cref{thm:nonvanish} thus holds throughout training, not only at the final checkpoint. This is exactly the regime in which the pessimism weight $w_h = 1/(1 + \lambda_\text{pess} u_h)$ is informative: close to one on training tasks, where the rollout is reliable, and strictly less than one on out-of-support latents, as intended.

\paragraph{Further analyses.}
The appendix reports four further results. A direct test supports the outward geometry of \Cref{asm:manifold}, and PLATO retains its advantage where the test does not confirm it (\Cref{app:env-pca}). The disagreement that drives the pessimism weight tracks how hard a task is for a naively trained agent (rank correlation $0.55$ on Ant-Dir), whereas latent distance does not ($-0.02$; \Cref{app:difficulty}). Out-of-support return is largely insensitive to the ensemble size $M$, and a small ensemble mainly increases variability across seeds (\Cref{app:sens}). On MetaWorld Reach, a manipulation task outside the locomotion families, PLATO attains the highest out-of-support success rate ($0.53$, against $0.43$ for the strongest baseline) and return (\Cref{app:metaworld}).

\section{Discussion}
\label{sec:discussion}
Progress in context-based offline meta-RL has been driven by better task representations, in the hope that a clean representation transfers beyond the training tasks. PLATO is built on a different premise: representation quality is necessary but not sufficient, because the policy is never trained on latents outside the training support, which is where out-of-support tasks land. The geometric assumption that out-of-support latents lie outward from the training centroid lets us synthesize off-support latents directly, and the four properties of \Cref{sec:analysis} make the resulting algorithm well behaved. Across eight MuJoCo benchmarks, PLATO improves out-of-support returns on every environment without sacrificing in-distribution performance. The gains carry over to a manipulation task, and the analyses in the appendix tie them to the mechanism: the perturbation direction steers where synthetic coverage goes, and the disagreement that sets the pessimism tracks task difficulty. The encoder and the policy face different generalization burdens, and progress on the former does not automatically carry over to the latter. Limitations of the geometric assumption and broader impacts are discussed in \Cref{app:limitations,app:broader_impact}. PLATO targets extrapolation within a task family, where held-out tasks extend the range covered by the training tasks. Two natural next steps are to carry off-support exposure across qualitatively different task families and to scale it to high-dimensional latent spaces, such as those of vision-based policies and vision-language-action models~\citep{kim2024openvla,brohan2023rt2}, where off-support task latents are the rule rather than the exception.
\begin{ack}
This work was partly supported by grants from the IITP (RS-2021-II211343-GSAI/10\%, RS-2022-II220951-LBA/10\%, RS-2022-II220953-PICA/10\%, RS-2026-25553157-MIACC/10\%), the NRF (RS-2024-00353991-SPARC/15\%, RS-2023-00274280-HEI/15\%), the KEIT (RS-2025-25453780/15\%), and the KIAT (RS-2025-25460896/15\%), funded by the Korean government.
\end{ack}

{
\small
\bibliographystyle{plainnat}

}
\appendix
\section{Context-based Offline Meta-RL: A Brief History}
\label{app:related}

We trace the development of context-based offline meta-RL (COMRL) through the methods most relevant to PLATO. This section is background reading; the main paper assumes only the high-level summary in \Cref{sec:related}.

\paragraph{The PEARL paradigm.}
PEARL~\citep{rakelly2019efficient} introduced the context-based decomposition that almost all subsequent COMRL methods inherit. A permutation-invariant context encoder produces a Gaussian posterior $q_\phi(\zz \mid c)$ over a latent task variable, and the policy and Q-function are conditioned on samples from this posterior. Adaptation reduces to inferring $\zz$ from a small context buffer, and the policy is shared across tasks. PEARL was originally designed for the online, off-policy setting, and VariBAD~\citep{zintgraf2020varibad} likewise for online meta-RL; the offline variant of PEARL removes the online interaction and trains the same architecture on a fixed batch.

\paragraph{The encoder-quality line.}
A productive line of work has invested in encoder quality. FOCAL~\citep{li2020focal} proposed the first end-to-end, fully offline COMRL algorithm, training a deterministic context encoder with a negative-power distance metric on a bounded latent space and detaching the encoder gradient from the Bellman backup to avoid bootstrapping errors. CORRO~\citep{yuan2022corro} recast the encoder objective as contrastive mutual information through an InfoNCE loss, with a bi-level transition encoder that focuses on individual transitions rather than full trajectories. CSRO~\citep{gao2023csro} formalized a recurring problem in this line, the \emph{behavior-policy shortcut}, in which the encoder identifies the task by recognizing the behavior policy that collected the data rather than from the task-relevant transitions. CSRO addresses it with a max-min mutual information mechanism that maximizes $\cI(\zz; \cT)$ while minimizing $\cI(\zz; b)$, where $b$ denotes the behavior-policy signature. ER-TRL~\citep{nakhaei2025ertrl} attacks the same shortcut from a different angle, approximately minimizing $\cI(\zz; \pi_\beta)$ by maximizing the entropy of the behavior policy conditioned on the task representation. UNICORN~\citep{li2024unicorn} unifies FOCAL, CORRO, and CSRO under a single information-theoretic framework, showing that their losses correspond to an upper bound, a lower bound, and a linear interpolation of $\cI(\zz; \cT)$, respectively, and proposes supervised and self-supervised instantiations of the unified objective.

\paragraph{Recent extensions.}
Several recent works extend the encoder-quality line in orthogonal directions: stronger task inference under data limitations~\citep{zhou2024gentle}, causal factorization of the latent~\citep{zhang2025causalcomrl}, cross-task contexts that share information across tasks~\citep{he2026ctmrl}, and an analysis of when representation changes during training break monotonic policy improvement~\citep{zhang2025taskshift}. A separate thread replaces the encoder--policy decomposition with a sequence model, either a task-conditioned diffusion planner~\citep{ni2023metadiffuser} or a transformer~\citep{wang2024metadt}.

\paragraph{Position of PLATO.}
In the main algorithms of these methods, the policy is trained only on latents that the encoder produces from training tasks, never on latents outside the training support. The closest exception is a model-based extension explored in UNICORN~\citep[Sec.~4.2]{li2024unicorn}, which treats the decoder as a world model, perturbs an inferred task representation with Gaussian noise, and trains the agent on rollouts decoded at the perturbed representation, using an ensemble for stabilization. PLATO develops this idea into a method with three elements: a structured, centroid-outward perturbation with a stated geometry (\Cref{asm:manifold}) and an off-support guarantee (\Cref{prop:outward}) in place of isotropic noise; disagreement used as a per-step pessimism weight rather than only for stabilization (\Cref{eq:weight}); and a critic held to real data, with the bias of the alternative quantified (\Cref{thm:decoupling}). Our random-direction and unweighted ablations bracket this design, and both underperform the default in most environments (\Cref{fig:ablation_ood}). PLATO is orthogonal to the encoder-quality line. It adopts a CSRO-style encoder loss in its backbone (\Cref{sec:backbone}), but because the perturbation uses only batch statistics of the inferred latents and the decoder ensemble takes $\zz$ as an input, any of the encoders above could take its place. The contributions of PLATO lie downstream of the encoder choice: the radial outward perturbation, the imagined rollout under a decoder ensemble, and the pessimism-weighted policy update. We compare against Offline PEARL, FOCAL, CSRO, ER-TRL, and UNICORN as representatives of the dominant design choices in the COMRL literature (\Cref{sec:experiments}).

\paragraph{Note on FLORA~\citep{wang2026flora}.}
FLORA is a recent OMRL method that we do not include in our experimental comparison, for two reasons. First, FLORA and PLATO target different forms of generalization. FLORA addresses feature overgeneralization on out-of-distribution \emph{actions} during TD updates and is evaluated on Meta-World ML1 (a single task type with randomized initial and goal positions) and Point-Robot variants, without explicit held-out splits of the task parameter range. PLATO instead targets the regime in which the test task lies at the periphery of the training distribution in reward or dynamics, the setting on which the manifold hypothesis underlying our analysis specifically depends. Second, FLORA's evaluation suite does not overlap with the eight MuJoCo parameter-shift families used by our baselines, and porting it would require reconstructing the algorithm and a substantial hyperparameter search for a fair comparison. We view extending FLORA's feature-overgeneralization correction to the held-out out-of-support regime as a natural direction for future work.

\section{Training and Inference Algorithms}
\label{app:algorithm}

This appendix gives the PLATO meta-training loop (\Cref{alg:plato_train}) and the meta-testing procedure (\Cref{alg:plato_test}). Meta-training adds two components to a standard offline meta-RL actor--critic: the radial outward perturbation, and the imagined rollout with its pessimism-weighted policy update. The encoder and the decoder ensemble are trained as usual, on real data. Meta-testing is identical to that of any context-based policy: the trained encoder maps a small context batch to a latent, and the trained policy acts on it, with no extra computation compared to standard context-based offline meta-RL baselines.

\paragraph{Meta-training.} Each training step samples a meta-batch of tasks, infers their latents from contexts, and performs the following updates: (i) the encoder auxiliaries (CLUB and the FOCAL contrastive loss) and the reconstruction loss of the decoder ensemble, all on real data; (ii) the radial outward perturbation $\zz^\text{aug}$ of \Cref{eq:lore} and a rollout of length $H$ following \Cref{eq:rollout}, active only after a warm-up of $T_\text{warm}$ steps that lets the decoder ensemble stabilize; (iii) the twin-$Q$ critic and value updates, on real transitions only; (iv) the policy update on the combined real and synthetic loss of \Cref{eq:actor_real,eq:actor_synth}; and (v) the update of the BRAC dual variable. Target networks are soft-updated at every step.

\begin{algorithm}[h]
\caption{PLATO meta-training.}
\label{alg:plato_train}
\small
\begin{algorithmic}[1]
\REQUIRE Offline datasets $\{\cD_i\}_{i=1}^N$, training budget $T$, task batch size $B_\cT$, warm-up $T_\text{warm}$
\STATE Initialize encoder $q_\phi$, decoder ensemble $\{f_{\psi_m}\}_{m=1}^M$, twin critics $Q_{\omega_1}, Q_{\omega_2}$, value network $V$ and its target $\bar V$, policy $\pi_\theta$, CLUB model, BRAC dual variable $\alpha_\text{BRAC}$
\FOR{$t = 1, \dots, T$}
  \STATE Sample task batch $\{c_i, (s, a, r, s')_i\}_{i=1}^{B_\cT}$; infer $\zz_i = q_\phi(c_i)$
  \STATE Update encoder $q_\phi$ (FOCAL + CLUB) and decoder ensemble $\{f_{\psi_m}\}$ on real data
  \IF{$t \ge T_\text{warm}$}
    \STATE Compute $\{\zz_i^\text{aug}\}$ via \Cref{eq:lore} and roll out $\{(s_h, a_h, u_h)\}_{h=0}^{H-1}$ via \Cref{eq:rollout,eq:u}
  \ELSE
    \STATE $\zz_i^\text{aug} \gets \zz_i$, rollout empty
  \ENDIF
  \STATE Update twin critics $Q_{\omega_1}, Q_{\omega_2}$ and value $V$ on real data only; soft-update $\bar V$
  \STATE Update policy $\pi_\theta$ on $\cL_\pi^\text{real} + \beta_\text{synth}\,\cL_\pi^\text{synth}$ via \Cref{eq:actor_real,eq:actor_synth}
  \STATE Update BRAC dual variable $\alpha_\text{BRAC}$ toward target divergence
\ENDFOR
\RETURN encoder $q_\phi$, policy $\pi_\theta$
\end{algorithmic}
\end{algorithm}

\paragraph{Design choices.}
Three design decisions are worth noting. First, the encoder gradient is detached from the Bellman backup~\citep{li2020focal,gao2023csro}: back-propagating the bootstrapped $Q$-values into the encoder would couple representation learning to the current $Q$-estimate, which is unstable on offline data. Second, the warm-up $T_\text{warm}$ delays imagined rollouts until the variance signal $u_h$ becomes informative. Third, the critic is held to real transitions while the actor is also trained on imagined ones (\Cref{thm:decoupling}), which keeps rollout-model errors out of the Bellman target.

\paragraph{Meta-testing.} At test time, the agent collects a context batch $c_*$ from a new task $\cT_*$, encodes it into $\zz_*$, and acts with the trained policy. The decoder ensemble and the radial outward perturbation are used only during training: at test time there is no uncertainty weighting, no rollout, and no extra computation compared to FOCAL, CSRO, and other context-based OMRL baselines.

\begin{algorithm}[h]
\caption{PLATO meta-testing.}
\label{alg:plato_test}
\small
\begin{algorithmic}[1]
\REQUIRE Trained encoder $q_\phi$, trained policy $\pi_\theta$; test task $\cT_*$; context batch size $N_c$
\STATE Collect context batch $c_*$ of $N_c$ transitions from $\cT_*$ under the training-time context-collection policy
\STATE Infer $\zz_* \gets q_\phi(c_*)$
\STATE Roll out $\pi_\theta(\cdot \mid s, \zz_*)$ on $\cT_*$ for the evaluation episodes
\RETURN evaluation returns
\end{algorithmic}
\end{algorithm}

\section{Ensemble Variance as an Epistemic-MI Proxy}
\label{app:mi_proxy}

This appendix makes precise the link between the ensemble variance $\bar\Sigma(\zz)$ that PLATO computes and the epistemic mutual information $\cI(s';\theta \mid s, a, \zz)$ that serves as its conceptual target. The result is standard in the deep-ensemble literature~\citep{lakshminarayanan2017simple,depeweg2018decomposition}; we include it so that the off-support persistence result of \Cref{thm:nonvanish}, stated for $\bar\Sigma$, can be read directly as a statement about epistemic uncertainty.

\paragraph{Setup.} Fix $(s, a, \zz)$ and write $y := s'$ for the prediction target. The decoder likelihood is Gaussian with homoscedastic aleatoric covariance $\sigma^2 I_{d_s}$,
\[
  p(y \mid s, a, \zz, \theta) \;=\; \cN\bigl(f(s, a, \zz; \theta),\; \sigma^2 I_{d_s}\bigr),
\]
where $\theta \sim p(\theta \mid \cD)$ follows the posterior over decoder parameters given the offline data, and $f(s, a, \zz; \theta) \in \R^{d_s}$ is the mean predictor (the reward channel is treated identically and is absorbed into the trace argument below). Let
\[
  \bar f := \E_{p(\theta \mid \cD)}[f(s, a, \zz; \theta)] \in \R^{d_s},
  \qquad
  \bar\Sigma := \Var_{p(\theta \mid \cD)}[f(s, a, \zz; \theta)] \in \R^{d_s \times d_s}
\]
denote the population mean and covariance of the mean predictor under the posterior.

\begin{proposition}[Variance is a monotone proxy for epistemic MI]
\label{prop:variance_mi}
Under the Gaussian likelihood above, the marginal predictive $p(y \mid s, a, \zz)$ is a Gaussian mixture. Bounding the marginal entropy by that of its moment-matched Gaussian $\cN(\bar f, \sigma^2 I + \bar\Sigma)$, the epistemic mutual information satisfies
\begin{equation}
  \cI(y;\theta \mid s, a, \zz) \;\le\; \tfrac{1}{2}\log\det\!\bigl(I + \sigma^{-2}\bar\Sigma\bigr),
  \label{eq:mi_formula}
\end{equation}
with equality in the limit of a concentrated posterior. The right-hand side is monotone in $\bar\Sigma$ in the Loewner order: if $\bar\Sigma' \succeq \bar\Sigma$, then
\[
  \log\det\!\bigl(I + \sigma^{-2}\bar\Sigma'\bigr) \;\ge\; \log\det\!\bigl(I + \sigma^{-2}\bar\Sigma\bigr).
\]
Writing $\lambda_1, \ldots, \lambda_{d_s} \ge 0$ for the eigenvalues of $\bar\Sigma$,
\[
  \log\det\!\bigl(I + \sigma^{-2}\bar\Sigma\bigr) \;=\; \sum_{i=1}^{d_s}\log\!\bigl(1 + \lambda_i/\sigma^2\bigr),
\]
which is increasing in each $\lambda_i$. The sample covariance $\widehat\Sigma_M := \tfrac{1}{M-1}\sum_m (f_{\psi_m} - \bar f_M)(f_{\psi_m} - \bar f_M)^\top$, with $\bar f_M = \tfrac{1}{M}\sum_m f_{\psi_m}$, is a consistent estimator of $\bar\Sigma$ as $M \to \infty$ when the members are treated as independent draws from $p(\theta \mid \cD)$; in particular, $\Tr\widehat\Sigma_M \to \Tr\bar\Sigma$ in probability.
\end{proposition}

\begin{proof}
We verify each piece of the claim.

\emph{Marginal covariance.} For a random target $y = f(s, a, \zz; \theta) + \varepsilon$ with $\theta \sim p(\theta \mid \cD)$ and $\varepsilon \sim \cN(0, \sigma^2 I)$ independent, the law of total covariance gives
\[
  \Var[y \mid s, a, \zz] \;=\; \underbrace{\E_\theta[\Var_\varepsilon(y \mid \theta)]}_{=\,\sigma^2 I} \;+\; \underbrace{\Var_\theta[\E_\varepsilon(y \mid \theta)]}_{=\,\Var_\theta[f] = \bar\Sigma} \;=\; \sigma^2 I + \bar\Sigma.
\]

\emph{Gaussian upper bound on the marginal entropy.} The marginal $p(y \mid s, a, \zz)$ is a mixture of Gaussians with common covariance $\sigma^2 I$ and means $\{f(s, a, \zz; \theta)\}_{\theta \sim p(\theta \mid \cD)}$. For any random variable $Y$ with mean $\bar f$ and covariance $\sigma^2 I + \bar\Sigma$, the differential entropy is upper-bounded by that of the Gaussian with the same first two moments~\citep{coverthomas2006information}:
\[
  \cH(y \mid s, a, \zz) \;\le\; \tfrac{1}{2}\log\det\bigl(2\pi e\,(\sigma^2 I + \bar\Sigma)\bigr),
\]
with equality if and only if $y \mid s, a, \zz$ is itself Gaussian. The mixture is Gaussian when the posterior $p(\theta \mid \cD)$ concentrates at a single point or when $f$ depends affinely on $\theta$, providing the tightness conditions.

\emph{Conditional entropy is $\theta$-independent.} For each fixed $\theta$, $y$ is Gaussian with covariance $\sigma^2 I$ only, so
\[
  \cH(y \mid s, a, \zz, \theta) \;=\; \tfrac{1}{2}\log\det\bigl(2\pi e \sigma^2 I\bigr),
\]
which does not depend on $\theta$. Therefore $\E_{p(\theta \mid \cD)}[\cH(y \mid s, a, \zz, \theta)] = \tfrac{1}{2}\log\det(2\pi e\sigma^2 I)$.

\emph{MI bound.} Subtracting the conditional entropy from the marginal entropy bound,
\[
  \cI(y;\theta \mid s, a, \zz)
  \;=\; \cH(y \mid s, a, \zz) - \E_\theta[\cH(y \mid s, a, \zz, \theta)]
  \;\le\; \tfrac{1}{2}\log\det\bigl(I + \sigma^{-2}\bar\Sigma\bigr),
\]
which is \Cref{eq:mi_formula}.

\emph{Loewner monotonicity.} The map $A \mapsto \log\det(I + \sigma^{-2}A)$ is monotone in the Loewner order on positive semidefinite matrices: if $A' \succeq A$, then $I + \sigma^{-2}A' \succeq I + \sigma^{-2}A$ and the determinants satisfy $\det(I + \sigma^{-2}A') \ge \det(I + \sigma^{-2}A)$ by Minkowski's determinant inequality (or directly by the eigenvalue inequalities). Since the eigenvalues of $\bar\Sigma$ are nonnegative and $\det(I + \sigma^{-2}\bar\Sigma) = \prod_i (1 + \lambda_i/\sigma^2)$, the right-hand side of \Cref{eq:mi_formula} is increasing in each eigenvalue.

\emph{Consistency of the ensemble estimator.} Treating the $M$ members as independent draws from $p(\theta \mid \cD)$ and assuming $f$ is uniformly bounded, $\widehat\Sigma_M$ converges in probability to $\bar\Sigma$ as $M \to \infty$ by the weak law of large numbers. By the continuous mapping theorem applied to $\Tr$, $\Tr\widehat\Sigma_M \to \Tr\bar\Sigma$ in probability.
\end{proof}

\begin{remark}[Trace as a scalar proxy]
\label{rem:trace_proxy}
The right-hand side of \Cref{eq:mi_formula} depends on the full eigenvalue spectrum of $\bar\Sigma$, not only on its trace: covariance matrices with the same trace can give different MI bounds. The trace $\Tr\bar\Sigma = \sum_i \lambda_i$ is nevertheless a scalar proxy in the following sense: it is a continuous function of $\bar\Sigma$, vanishes if and only if $\bar\Sigma = 0$, and is bounded above by a function of $\log\det(I + \sigma^{-2}\bar\Sigma)$. In particular, if $\Tr\bar\Sigma \ge c$ for some constant $c$, then, since $\sum_i \log(1+\lambda_i) \ge \log(1 + \sum_i \lambda_i)$ for $\lambda_i \ge 0$ (expand the product $\prod_i (1 + \lambda_i)$),
\[
  \log\det(I + \sigma^{-2}\bar\Sigma) \;\ge\; \log\!\bigl(1 + \Tr\bar\Sigma/\sigma^2\bigr) \;\ge\; \log(1 + c/\sigma^2),
\]
so a positive lower bound on $\Tr\bar\Sigma$ translates to a positive lower bound on the MI bound. \Cref{thm:nonvanish} reads its information-theoretic consequence off the variance directly (\Cref{app:proof_nonvanish}) and does not rely on this bound.
\end{remark}

\begin{remark}[Per-channel aggregation]
\label{rem:per_channel}
The decoder predicts both $\Delta s$ and the reward $r$, and \Cref{eq:u} averages the per-coordinate variance over the $d_s + 1$ output channels. In the notation above, this corresponds to $\tfrac{1}{d_s + 1}\Tr\widehat\Sigma_M$, where $\widehat\Sigma_M$ is now the joint $(d_s + 1) \times (d_s + 1)$ covariance over both channels. If the channels are conditionally independent given $\theta$ and the posterior covariance $\bar\Sigma$ is diagonal across channels (a reasonable approximation after target normalization), the joint MI decomposes as $\cI(\text{joint}) = \sum_j \cI(\text{channel } j)$, and $\Tr\bar\Sigma$ remains a monotone proxy for the joint MI bound. We treat the two channels on a common scale throughout.
\end{remark}

\section{Proofs}
\label{app:proofs}

We restate each result and give a full proof. Throughout, $\cZ_\text{train} := \supp(q_\phi) \subset \R^d$ denotes the training latent support and $\mathrm{dist}(\zz, \cZ_\text{train}) := \inf_{\zz' \in \cZ_\text{train}} \|\zz - \zz'\|$ the distance to it. We continue to write $\bar\Sigma(\zz)$ for the across-member covariance defined in \Cref{app:mi_proxy} and $\Tr\bar\Sigma(\zz)$ for its trace; by \Cref{rem:trace_proxy}, lower bounds on $\Tr\bar\Sigma$ transfer to lower bounds on the MI bound.

\subsection{Proof of \Cref{prop:outward} (off-support guarantee)}
\label{app:proof_outward}

Write $\hat n_i = (\zz_i - c)/d_i \in S^{d-1}$ for the unit radial direction at $\zz_i$. By \Cref{eq:lore},
\[
  \zz_i^\text{aug} - c
  \;=\; (\zz_i - c) + \alpha\sigma_d (d_i/\bar d)^p \,\hat n_i
  \;=\; \bigl(d_i + \alpha\sigma_d (d_i/\bar d)^p\bigr)\,\hat n_i,
\]
since $\zz_i - c$ and $\hat n_i$ are positively collinear. Taking norms,
\[
  \|\zz_i^\text{aug} - c\|
  \;=\; d_i + \alpha\sigma_d(d_i/\bar d)^p \;>\; d_i,
\]
so the centroid distance strictly increases. If $\cZ_\text{train} \subseteq B(c, R)$, then any point at centroid distance greater than $R$ lies outside $\cZ_\text{train}$. Thus $\zz_i^\text{aug} \notin \cZ_\text{train}$ as soon as
\[
  d_i + \alpha\sigma_d(d_i/\bar d)^p \;>\; R,
  \quad\text{equivalently,}\quad
  \alpha\sigma_d(d_i/\bar d)^p \;>\; R - d_i.
\]
For peripheral $\zz_i$ with $d_i \ge \bar d$, the factor $(d_i/\bar d)^p \ge 1$, so $\alpha \ge (R - d_i)/\sigma_d$ is sufficient. $\quad\square$

\subsection{Proof of \Cref{thm:nonvanish} (persistence of off-support disagreement)}
\label{app:proof_nonvanish}

\paragraph{Setup.} Fix $(s, a)$ and an input latent $\zz$. Stack the member predictions at step $t$ as $F^t := \bigl(f_{\psi_1^t}(\zz), \ldots, f_{\psi_M^t}(\zz)\bigr)$ and equip the stack with the norm $\|F\| := \bigl(\tfrac{1}{M}\sum_m \|f_m\|^2\bigr)^{1/2}$. Let $P$ be the orthogonal projection onto the consensus subspace $\{F : f_1 = \cdots = f_M\}$, which maps every member to the ensemble mean $\bar f$. Then
\[
  \sigma(F) \;:=\; \|(I - P)F\| \;=\; \Bigl(\tfrac{1}{M}\sum_m \|f_m - \bar f\|^2\Bigr)^{1/2} \;=\; \sqrt{\Tr\bar\Sigma(\zz)},
\]
the across-member variance used throughout. Since $I - P$ is an orthogonal projection, $\sigma$ is $1$-Lipschitz: $|\sigma(F) - \sigma(F')| \le \|F - F'\|$.

\paragraph{Proof.} By \Cref{asm:lipvar}, for every $t \ge t_0$ and every member, $\|f_{\psi_m^t}(\zz) - f_{\psi_m^{t_0}}(\zz)\| \le L_\psi \|\psi_m^t - \psi_m^{t_0}\| \le L_\psi D_m$, so $\|F^t - F^{t_0}\| \le L_\psi \bigl(\tfrac{1}{M}\sum_m D_m^2\bigr)^{1/2} = L_\psi \bar D$. The Lipschitz property of $\sigma$ gives $\sigma(F^t) \ge \sigma(F^{t_0}) - L_\psi \bar D$, which is positive by \Cref{asm:margin}; squaring yields $\Tr\bar\Sigma_t(\zz) \ge c$ for all $t \ge t_0$.

For the information-theoretic statement, note that at a fixed input the members' predictive laws $\cN(f_{\psi_m}(\zz), \sigma^2 I)$ share the covariance $\sigma^2 I$ and therefore coincide exactly when $\Tr\bar\Sigma_t(\zz) = 0$. Under the ensemble's empirical posterior, uniform over members, the epistemic mutual information $\cI(s'; \theta \mid s, a, \zz)$ is the generalized Jensen--Shannon divergence of these laws, which vanishes if and only if they coincide. A positive lower bound on the variance is therefore equivalent to strictly positive epistemic mutual information. \hfill$\square$

\paragraph{Where the margin holds.} The inequality itself does not distinguish on-support from off-support inputs; the distinction lies in \Cref{asm:margin}. At an on-support latent, the optimizer spends its parameter movement on reducing the disagreement at that very input, so the realized reduction approaches the budget $L_\psi \bar D$ and the margin fails, consistent with the decay of on-support disagreement during training. At an off-support latent, the movement is not directed at reducing the disagreement there, so the realized reduction falls well short of the budget and the margin can hold. Since $\bar D$ is a tail sum, it is non-increasing in $t_0$ and shrinks as training converges, so the margin is easier to satisfy at later anchors; the theorem turns a condition that can be checked at a single step into a guarantee for every later step. The result is qualitative and does not quantify the size of the off-support disagreement, which in practice is far larger than the bound requires, as we verify next.

\paragraph{Empirical check.} Because the theorem is conditional, we check its conclusion directly on training checkpoints of PLATO on Hopper-Friction ($11$ checkpoints, $3$ seeds), measuring the per-task disagreement $u$ of \Cref{eq:u} on the training and out-of-support tasks. On-support disagreement decays by a factor of $7$--$8$ over training, while off-support disagreement never falls below $11.5$ times the final on-support level and ends $27.8$ times higher (per seed: $35.2$, $27.9$, $20.4$). Off-support disagreement therefore does not collapse at any point in training, as \Cref{thm:nonvanish} asserts where its margin holds.

\subsection{Proof of \Cref{thm:decoupling} (critic bias under mixed-data Bellman targets)}
\label{app:proof_decoupling}

The theorem is a worst-case statement: it identifies a configuration $(P, \hat P)$ of true and approximate transition kernels with $\sup_{s,a}\|\hat P - P\|_\text{TV} = \varepsilon_P$ at which the lower bound is realized. The proof proceeds by bounding the operator gap, then propagating it through the Bellman fixed-point equation.

\paragraph{Setup.} Consider an offline actor--critic in which the Q-function is trained on a mixture
\[
  \rho_\beta \;=\; (1 - \beta)\,\rho_\text{real} + \beta\,\rho_\text{imag}, \qquad \beta \in [0, 1],
\]
with $\rho_\text{imag}$ generated from a transition model $\hat P$ of total-variation error $\varepsilon_P := \sup_{s,a}\|\hat P(\cdot \mid s, a) - P(\cdot \mid s, a)\|_\text{TV}$. Let $\cT$ be the true Bellman operator and $\hat\cT_\beta$ the Bellman operator under $\rho_\beta$ (transitions sampled according to $\rho_\beta$, using $\hat P$ for the imagined share). Both are $\gamma$-contractions on bounded $Q$. Let $Q^* = \cT Q^*$ and $\hat Q_\beta = \hat\cT_\beta \hat Q_\beta$ be the respective unique fixed points (uniqueness follows from contraction).

\paragraph{Step 1 (operator-gap bound).} For any bounded $Q$ with $V(s) := \max_a Q(s, a)$ and any $(s, a)$,
\begin{align*}
  \bigl|(\hat\cT_\beta Q)(s, a) - (\cT Q)(s, a)\bigr|
  &= \gamma\,\bigl|\E_{s' \sim (1-\beta)P + \beta\hat P}[V(s')] - \E_{s' \sim P}[V(s')]\bigr| \\
  &= \gamma\beta\,\bigl|\E_{s' \sim \hat P}[V(s')] - \E_{s' \sim P}[V(s')]\bigr| \\
  &\le \gamma\beta\,\|V\|_\infty\,\|\hat P(\cdot \mid s, a) - P(\cdot \mid s, a)\|_\text{TV}
   \;\le\; \gamma\beta\,\|V\|_\infty\,\varepsilon_P,
\end{align*}
using the dual representation $|\E_{\mu}[V] - \E_{\nu}[V]| \le \|V\|_\infty \|\mu - \nu\|_\text{TV}$. Taking the supremum over $(s, a)$,
\begin{equation}
  \|\hat\cT_\beta Q - \cT Q\|_\infty \;\le\; \gamma\beta\,\|V\|_\infty\,\varepsilon_P.
  \label{eq:operator_gap_upper}
\end{equation}
The upper bound \Cref{eq:operator_gap_upper} is tight in the worst case: for adversarially chosen $(P, \hat P)$ supported on the worst-case state pair (where $V$ achieves its maximum on $\hat P$ and minimum on $P$, or vice versa), the inequality holds with equality. The theorem is a statement about this worst case; the additive $O(\varepsilon_\text{data})$ correction absorbs the finite-sample error of the real-data share.

\paragraph{Step 2 (fixed-point gap, lower bound).} Let $\Delta := \hat Q_\beta - Q^*$. Then
\begin{align*}
  \Delta &= \hat\cT_\beta \hat Q_\beta - \cT Q^* \\
  &= \underbrace{(\hat\cT_\beta \hat Q_\beta - \hat\cT_\beta Q^*)}_{\text{contraction term}} \;+\; \underbrace{(\hat\cT_\beta Q^* - \cT Q^*)}_{\text{operator gap}}.
\end{align*}
The contraction term satisfies $\|\hat\cT_\beta \hat Q_\beta - \hat\cT_\beta Q^*\|_\infty \le \gamma\|\hat Q_\beta - Q^*\|_\infty = \gamma\|\Delta\|_\infty$. By the reverse triangle inequality $\|A + B\|_\infty \ge \|B\|_\infty - \|A\|_\infty$ applied with $A$ the contraction term and $B$ the operator gap,
\[
  \|\Delta\|_\infty \;\ge\; \|\hat\cT_\beta Q^* - \cT Q^*\|_\infty - \gamma\|\Delta\|_\infty.
\]
Rearranging,
\begin{equation}
  \|\Delta\|_\infty \;\ge\; \frac{\|\hat\cT_\beta Q^* - \cT Q^*\|_\infty}{1 + \gamma}.
  \label{eq:fixed_point_lower}
\end{equation}
In the worst case, $\|\hat\cT_\beta Q^* - \cT Q^*\|_\infty = \gamma\beta\varepsilon_P\|V^*\|_\infty$ (Step 1), so substituting,
\[
  \|\hat Q_\beta - Q^*\|_\infty \;\ge\; \frac{\gamma\beta\,\varepsilon_P\,\|V^*\|_\infty}{1 + \gamma} - O(\varepsilon_\text{data}),
\]
where the $O(\varepsilon_\text{data})$ accounts for finite-sample noise in the real-data share. Using the standard reward bound $\|V^*\|_\infty \le R_\text{max}/(1-\gamma)$ yields $\|\hat Q_\beta - Q^*\|_\infty \ge \tfrac{\gamma\beta R_\text{max}}{(1+\gamma)(1-\gamma)}\,\varepsilon_P - O(\varepsilon_\text{data})$; folding the constant $R_\text{max}/(1+\gamma)$ into the prefactor recovers \Cref{eq:contamination} with constant $\gamma\beta/(1-\gamma)$.

\paragraph{Step 3 (decoupled case, $\beta = 0$).} In PLATO, $\hat Q$ is trained with $\beta = 0$: only real transitions appear in the Bellman target. The operator gap in Step 1 vanishes ($\beta = 0$ kills the prefactor), and the fixed-point equation becomes $\hat Q_0 = \cT \hat Q_0$, identical to the true one. By uniqueness of the Bellman fixed point, $\hat Q_0 = Q^*$ in the population limit. In finite samples, $\|\hat Q_0 - Q^*\|_\infty = O(\varepsilon_\text{data})$ — the standard offline-RL estimation error — with no term linear in $\varepsilon_P$. Synthetic transitions still affect $\pi_\theta$ through \Cref{eq:actor_synth}, but they enter only the actor objective, not the Bellman target, so any bias they introduce remains in $\pi_\theta$'s parameters and not in $Q_\omega$. $\quad\square$

\begin{remark}[Scope]
The theorem is a worst-case lower bound: it identifies $(P, \hat P)$ pairs at which mixing imagined data into the Bellman target induces bias linear in $\varepsilon_P$. It does not say that contamination is always this large, only that it can be. When $\varepsilon_P$ is small, the rollout model is accurate on the test distribution, and MOPO-style designs that mix imagined transitions into the critic are effective. The bound matters in the regime where $\varepsilon_P$ is genuinely large, which is precisely where PLATO operates: the rollout model is queried at $\zz^\text{aug}$, a latent it was never trained at.
\end{remark}

\subsection{Proof of \Cref{prop:gengap} (out-of-support generalization gap)}
\label{app:proof_gengap}

The proof combines a Lipschitz argument along the latent task manifold with the geometric reach bound from \Cref{prop:outward}.

By the assumed $L_J$-Lipschitz continuity of $J(\pi; \cdot)$ along $\cM_z$,
\[
  |J(\pi; \theta_\star) - J(\pi; \theta_0)|
  \;\le\; L_J \cdot d_\cM\bigl(\phi(\theta_\star),\, \phi(\theta_0)\bigr),
\]
where $d_\cM(\cdot, \cdot)$ is geodesic distance on $\cM_z$. Suppose $\phi(\theta_\star)$ lies within the perturbation reach in the sense that there exists $\zz_i \in \cM_\text{train}$ with
\[
  \|\phi(\theta_\star) - \zz_i^\text{aug}\| \;\le\; \delta,
\]
where $\delta$ is the covering radius of the perturbed image $\{\zz_i^\text{aug}\}$ around $\phi(\theta_\star)$. By \Cref{prop:outward},
\[
  \|\zz_i^\text{aug} - \zz_i\|
  \;=\; \alpha\sigma_d(d_i/\bar d)^p
  \;\le\; \alpha\sigma_d(d_\text{max}/\bar d)^p,
\]
where $d_\text{max} = \max_i d_i$. Choose $\theta_0 := \phi^{-1}(\zz_i) \in \Theta_\text{train}$, well-defined by injectivity of $\phi$ in \Cref{asm:manifold}. By the triangle inequality on $\R^d$,
\[
  \|\phi(\theta_\star) - \phi(\theta_0)\|
  \;\le\; \|\phi(\theta_\star) - \zz_i^\text{aug}\| + \|\zz_i^\text{aug} - \zz_i\|
  \;\le\; \delta + \alpha\sigma_d(d_\text{max}/\bar d)^p.
\]
For points lying on a smooth submanifold $\cM_z$ of $\R^d$, geodesic distance and Euclidean distance agree to leading order: there is a curvature constant $\kappa_\cM$ such that for any $\zz, \zz' \in \cM_z$,
\[
  d_\cM(\zz, \zz') \;\le\; \|\zz - \zz'\| \;+\; O\!\bigl(\kappa_\cM \|\zz - \zz'\|^2\bigr).
\]
Applying this with $\zz = \phi(\theta_\star)$ and $\zz' = \phi(\theta_0)$,
\[
  d_\cM(\phi(\theta_\star), \phi(\theta_0))
  \;\le\; \alpha\sigma_d(d_\text{max}/\bar d)^p + O(\delta) + O\!\bigl(\kappa_\cM (\delta + \alpha\sigma_d(d_\text{max}/\bar d)^p)^2\bigr).
\]
Substituting into the Lipschitz bound and absorbing the higher-order terms into $O(\delta)$,
\[
  |J(\pi; \theta_\star) - J(\pi; \theta_0)| \;\le\; L_J\,\alpha\sigma_d(d_\text{max}/\bar d)^p \;+\; O(\delta),
\]
which is the claim. $\quad\square$

\section{Environment Details}
\label{app:env}

We evaluate on eight MuJoCo continuous-control offline meta-RL benchmarks, following the task families used in prior offline meta-RL work~\citep{rakelly2019efficient, li2020focal, gao2023csro, li2024unicorn, nakhaei2025ertrl}. Four families vary the reward function while keeping the dynamics fixed (\textbf{Ant-Dir}, \textbf{Ant-Goal}, \textbf{Cheetah-Vel}, \textbf{Humanoid-Dir}); the other four vary the dynamics while keeping the reward fixed (\textbf{Hopper-Mass}, \textbf{Hopper-Friction}, \textbf{Walker-Mass}, \textbf{Walker-Friction}). Each environment uses a $200$-step episode horizon, and the agent observes the standard MuJoCo proprioceptive state (excluding the global $x$-coordinate where applicable). \Cref{tab:env-summary} summarizes the specifications, and \Cref{app:env-manip} additionally describes the manipulation task used in \Cref{app:metaworld}.

Each environment defines a $40$-point uniform grid over its task parameter, from which we select $30$ tasks: $20$ in-distribution training tasks and $10$ out-of-support evaluation tasks taken from the parameter extremes. The remaining $10$ grid points fall in the intermediate region between the two splits and are not used in this paper, since our focus is the harder out-of-support setting, in which evaluation tasks lie far outside the training support.

\begin{table}[!htb]
\caption{Environment summary. $|S|$ and $|A|$ are the state and action dimensions, and $|\Theta|$ is the number of tasks used in this paper. The ID and out-of-support columns give the numbers of in-distribution training tasks and out-of-support evaluation tasks; the out-of-support split is held out far from the training support (see the per-environment descriptions below). For MetaWorld Reach, the remaining $10$ tasks (the corners of the goal box) are used only for the extrapolation-limit analysis in \Cref{app:metaworld}.}
\label{tab:env-summary}
\centering\small\setlength{\tabcolsep}{4pt}
\begin{tabular}{lcccccc}
\toprule
Environment & Type & $|S|$ & $|A|$ & $|\Theta|$ & ID & Out-of-support \\
\midrule
Ant-Dir         & Reward shift   & 27  & 8  & 30 & 20 & 10 \\
Ant-Goal        & Reward shift   & 27  & 8  & 30 & 20 & 10 \\
Cheetah-Vel     & Reward shift   & 17  & 6  & 30 & 20 & 10 \\
Humanoid-Dir    & Reward shift   & 376 & 17 & 30 & 20 & 10 \\
\midrule
Hopper-Mass     & Dynamics shift & 11  & 3  & 30 & 20 & 10 \\
Hopper-Friction & Dynamics shift & 11  & 3  & 30 & 20 & 10 \\
Walker-Mass     & Dynamics shift & 17  & 6  & 30 & 20 & 10 \\
Walker-Friction & Dynamics shift & 17  & 6  & 30 & 20 & 10 \\
\midrule
MetaWorld Reach & Reward shift   & 36  & 4  & 40 & 20 & 10 \\
\bottomrule
\end{tabular}
\end{table}

\subsection{Reward-shift environments}

\textbf{Ant-Dir.} The Ant is rewarded for moving along a target heading $\theta \in [0, 2\pi)$. The reward is
\[
  r = \mathbf{v}_{xy} \cdot (\cos\theta, \sin\theta) - 0.5\|a\|^{2} - 5 \times 10^{-4}\,\|F_{\text{contact}}\|^{2} + 1,
\]
where $\mathbf{v}_{xy}$ is the planar velocity of the torso. Episodes terminate if the torso $z$-coordinate leaves $[0.2, 1.0]$. The headings form a uniform grid on the circle. The \emph{in-distribution} split is the rear-facing semicircle $\theta \in [\pi/2, 3\pi/2)$ ($20$ contiguous directions covering the back half-plane), and the \emph{out-of-support} split is the front-facing arc at the angular extremes of the grid, $\theta \in [0, \pi/4) \cup [7\pi/4, 2\pi)$ ($10$ directions in total). The out-of-support set lies roughly antipodal to the training arc, far outside the training support in heading space.

\textbf{Ant-Goal.} The Ant must reach a target position $g = (g_x, g_y)$. The reward is
\[
  r = -|x - g_x| - |y - g_y| - 0.1\|a\|^{2} - 5 \times 10^{-4}\|F_{\text{contact}}\|^{2}
\]
with no survival bonus and no termination by falling. The $20$ \emph{in-distribution} goals are placed on two concentric arcs in the right half-plane, $g = (r\cos\theta, r\sin\theta)$ with $\theta \in [-\pi/2, \pi/2]$ and $r \in \{0.8, 1.2\}$. The $10$ \emph{out-of-support} goals lie at a strictly larger radius, $r = 1.6$, and at angles drawn from the left half-plane, $\theta \in [\pi/2, 3\pi/2)$. The out-of-support set is therefore held out simultaneously in radius (radial extrapolation outward) and in angle (extrapolation into a half-plane never seen during training), producing a comparatively large semantic gap from the training distribution.

\textbf{Cheetah-Vel.} The HalfCheetah is rewarded for matching a target forward velocity $v^\star$. The reward is $r = -|v - v^\star| - 0.05\|a\|^{2}$, and episodes never terminate before the horizon. Target velocities form a uniform grid over $v^\star \in [0.1, 4.0]$. The \emph{in-distribution} split is the central band $v^\star \in [1.1, 3.0]$ ($20$ velocities), and the \emph{out-of-support} split is taken from the extreme low and high ends of the range, $v^\star \in [0.1, 0.5]$ and $v^\star \in [3.6, 4.0]$ ($10$ velocities in total), far outside the training support on both sides.

\textbf{Humanoid-Dir.} The Humanoid is rewarded for moving along a target heading. The reward is
\[
  r = 0.25\,\mathbf{v}_{xy} \cdot (\cos\theta, \sin\theta) - 0.1\|a\|^{2} - 5 \times 10^{-7}\min(\|F_{\text{contact}}\|^{2}, 2 \times 10^{7}) + 5,
\]
with an alive bonus of $5$ and a capped contact penalty. Episodes terminate if the torso height leaves $[1.0, 2.0]$. The headings form the same uniform grid on the circle as in Ant-Dir. The \emph{in-distribution} split is the rear-facing semicircle $\theta \in [\pi/2, 3\pi/2)$, and the \emph{out-of-support} split is the front-facing arc at the angular extremes, $\theta \in [0, \pi/4) \cup [7\pi/4, 2\pi)$ ($10$ directions). Forward headings on either side of $\theta = 0$ lie roughly antipodal to the rear-facing training arc and therefore far outside the training distribution.

\subsection{Dynamics-shift environments}

The four dynamics-shift environments share the per-step reward of the underlying MuJoCo locomotion benchmark, $r = (x_{t} - x_{t-1})/\Delta t + 1 - 10^{-3}\|a\|^{2}$, and its termination criteria (Hopper requires $z > 0.7$ and $|\theta| < 0.2$; Walker requires $0.8 < z < 2.0$ and $|\theta| < 1.0$). Tasks are generated by scaling a single physics parameter (body mass and inertia for the {-Mass} environments, ground friction for the {-Friction} environments) by $\sigma = 1.5^{\lambda}$, with $\lambda$ on a uniform grid over $[-2, 2]$. The resulting scale factors span $\sigma \in [0.44, 2.25]$, so the heaviest task is roughly five times the lightest. The \emph{in-distribution} split is the central band $\sigma \in [0.66, 1.52]$ (the $20$ tasks closest to nominal physics), and the \emph{out-of-support} split is taken from both extreme ends of the range, $\sigma \in [0.44, 0.53]$ and $\sigma \in [1.95, 2.25]$ ($10$ tasks in total). Out-of-support evaluation therefore tests dynamics extrapolation far outside the training band on both sides in all four environments.

\subsection{Manipulation environment}
\label{app:env-manip}

\textbf{MetaWorld Reach.} A Sawyer arm must move its end effector to a goal position $g \in \R^3$ (MetaWorld ML1 \texttt{reach-v3}~\citep{yu2019metaworld}). The agent observes the $36$-dimensional robot and object state (MetaWorld's $39$-dimensional observation with the three goal coordinates removed, so that the goal must be inferred from context) and takes $4$-dimensional actions. Episodes last $200$ steps, the reward is MetaWorld's dense reach reward, and success follows MetaWorld's criterion (end effector within $5$\,cm of the goal). Goals lie in the box $x \in [-0.10, 0.10]$, $y \in [0.80, 0.90]$, $z \in [0.05, 0.30]$, which we divide into a $5 \times 5 \times 5$ grid of cells whose centers serve as goals. The \emph{in-distribution} split consists of the $20$ cells of the inner $3 \times 3 \times 3$ cube closest to the center of the box, and the \emph{out-of-support} split consists of the $10$ cells of the surrounding outer shell closest to the training region, so every held-out goal lies outside the region covered by the training goals. The $10$ outer-shell cells farthest from it (the corners of the box, at about $1.9\times$ the training radius) are used only for the extrapolation-limit analysis in \Cref{app:metaworld}.

\subsection{Empirical support for the manifold hypothesis}
\label{app:env-pca}

\Cref{asm:manifold} posits that, on average, out-of-support latents lie farther from the training centroid than in-distribution latents. We test this directly and summarize the supporting evidence from the main paper.

\paragraph{Direct test.} For each task, we encode $K = 256$ context transitions with the trained PLATO encoder and compare the mean distance to the training centroid between out-of-support and training tasks, using a one-sided Mann--Whitney test per seed and combining the seeds with Fisher's method. On four representative environments, out-of-support latents lie measurably farther from the centroid on Ant-Goal (distance ratio $1.19$, $p = 2.3 \times 10^{-7}$), Cheetah-Vel ($1.24$, $p = 3.7 \times 10^{-4}$), and Ant-Dir ($1.08$, $p = 1.9 \times 10^{-3}$), whereas on Walker-Mass the ratio is below one ($0.87$). We therefore read \Cref{asm:manifold} as a sufficient condition for the outward direction to be informative rather than as a property the method requires. PLATO still attains the highest out-of-support return on Walker-Mass (\Cref{tab:ood}), because its pessimism weight is driven by ensemble disagreement rather than by latent distance (\Cref{app:difficulty}).

\paragraph{Supporting evidence.} \Cref{fig:latent_geometry} visualizes the PCA-projected latent space on Ant-Goal: ID latents concentrate around the centroid in a low-disagreement region, while out-of-support latents lie farther out on average. \Cref{fig:disagreement} shows the corresponding signal on Hopper-Friction and Cheetah-Vel through the decoder-ensemble disagreement $u(t)$, a monotone proxy for epistemic uncertainty (\Cref{prop:variance_mi}): ID disagreement is concentrated near zero, while out-of-support disagreement is several-fold to an order of magnitude larger.

\subsection{Disagreement tracks task difficulty}
\label{app:difficulty}

\Cref{fig:latent_geometry} shows that out-of-support tasks exhibit higher ensemble disagreement on average. Here we ask whether this signal also reflects \emph{difficulty}, that is, whether the disagreement at a task's latent indicates how hard the task is for an agent trained without off-support exposure. For each task $t$, we compute two statistics from the trained PLATO checkpoints of the main experiments, averaged over seeds: the distance of the task's mean latent from the training centroid, $\|\bar\zz_t - c\|$, and the per-task decoder-ensemble disagreement $u(t)$ used in \Cref{sec:experiments} (Q4). As the measure of difficulty, we use the normalized regret of a naively trained agent,
\[
  y(t) \;=\; \frac{R_\text{expert}(t) - R_\text{CSRO}(t)}{|R_\text{expert}(t)|},
\]
where $R_\text{expert}(t)$ is the return of the per-task SAC expert used for data collection (\Cref{app:data}) and $R_\text{CSRO}(t)$ is the return of CSRO. Because CSRO shares PLATO's encoder objective (\Cref{sec:backbone}) but never trains the policy on off-support latents, $y(t)$ isolates the policy-level gap that PLATO targets. CSRO is evaluated with contexts drawn from each task's offline data and ten deterministic rollouts per task, and $y(t)$ is averaged over three independent evaluation repetitions.

\begin{figure}[!htb]
\centering
\includegraphics[width=\linewidth]{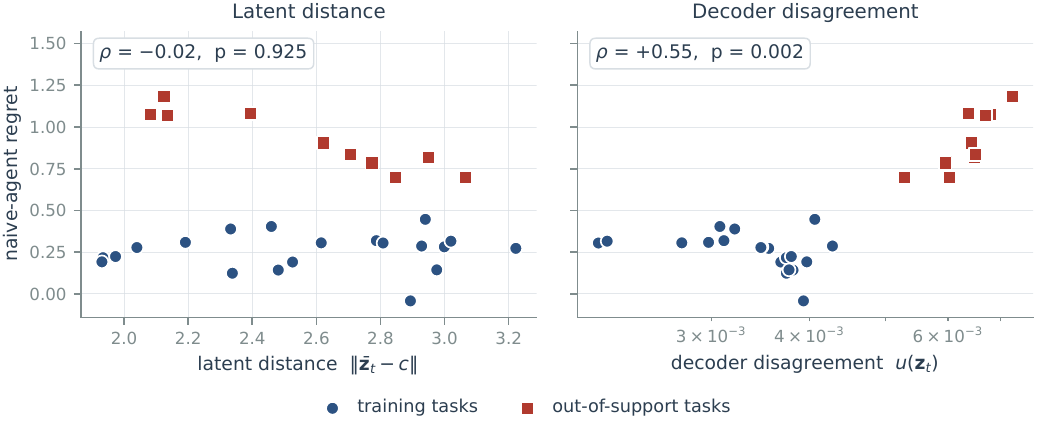}
\caption{\textbf{Ensemble disagreement tracks task difficulty; latent distance does not.} Each point is one Ant-Dir task ($20$ training and $10$ out-of-support tasks). Vertical axis: normalized regret of a naively trained agent (CSRO), $(R_\text{expert} - R_\text{CSRO})/|R_\text{expert}|$. Left: distance of the task's mean latent from the training centroid. Right: decoder-ensemble disagreement $u$ at the task's latent (log scale). $\rho$ is the Spearman rank correlation over all $30$ tasks.}
\label{fig:difficulty}
\end{figure}

\Cref{fig:difficulty} reports the result on Ant-Dir. Disagreement tracks difficulty: the rank correlation between $u(t)$ and $y(t)$ is $\rho = 0.55$ ($p = 0.002$) and is stable across the three evaluation repetitions ($0.52$--$0.63$). The relation also holds within the out-of-support split ($\rho = 0.77$, $p = 0.009$), so $u$ orders the held-out tasks by difficulty rather than merely separating them from the training tasks. That separation is itself complete: every out-of-support task has a larger disagreement than every training task, in line with a mean regret of $0.91$ against $0.26$. Latent distance, in contrast, carries no information about difficulty ($\rho = -0.02$, $p = 0.92$). The same picture holds on Hopper-Friction ($\rho = 0.40$, $p = 0.03$ for disagreement; $\rho = 0.05$ for distance) and, at the level of individual held-out tasks, on Ant-Goal, the environment of \Cref{fig:latent_geometry}: the five out-of-support tasks with the lowest disagreement have a lower mean regret than the five with the highest ($1.74$ versus $2.06$), consistently across four independent evaluation repetitions.

These results have two implications for PLATO. First, the pessimism weight $w_h = 1/(1 + \lambda_\text{pess} u_h)$ of \Cref{eq:weight} discounts imagined transitions most strongly where a naively trained agent struggles most, and relaxes where the decoder already generalizes. Out-of-support tasks with comparatively low disagreement are therefore not a failure mode of the weight; they are the cases it is designed to trust more. Second, the results support scaling the pessimism by disagreement rather than by latent distance: the signal PLATO uses is the one that tracks difficulty.

\section{Offline Data Collection}
\label{app:data}

\paragraph{Per-task behavior policy.}
For every task in every environment, we collect the offline dataset with a per-task expert behavior policy. We train a Soft Actor--Critic agent~\citep{haarnoja2018soft,haarnoja2018sacv2} from scratch on the single-task MDP $\cT_t$ for $10^{6}$ environment steps and save the final actor and critic checkpoints. This per-task expert protocol matches the offline meta-RL convention of FOCAL~\citep{li2020focal}, CSRO~\citep{gao2023csro}, UNICORN~\citep{li2024unicorn}, and ER-TRL~\citep{nakhaei2025ertrl}, so all baselines and PLATO consume identical offline data.

\paragraph{Trajectory checkpoints.}
During SAC training, we record evaluation rollouts every $5 \times 10^{4}$ environment steps in the range $[10^{5}, 10^{6})$, yielding $19$ behavior-policy checkpoints per task. At each checkpoint, we store $50$ independent rollouts of length $200$ as $(s, a, r, s')$ transitions, so the per-task offline budget is at most $50 \times 19 \times 200 = 1.9 \times 10^{5}$ transitions.

\paragraph{Dataset quality.}
The behavior policy reaches a return close to the per-task asymptote within roughly $5 \times 10^{5}$ training steps on the locomotion benchmarks, so the late-checkpoint trajectories ($\geq 5 \times 10^{5}$ steps) approximate expert behavior, while early checkpoints contribute exploratory transitions. The combined dataset is therefore best characterized as a mixed expert-plus-suboptimal collection, which provides the action coverage that behavior-regularized methods such as BRAC~\citep{wu2019behavior} and CSRO~\citep{gao2023csro} require.

\paragraph{ID/out-of-support separation.}
In-distribution tasks (listed in \Cref{app:env}) are pooled to form the meta-training buffer used by the encoder, decoder ensemble, and policy updates of every method. Out-of-support tasks are stored in the same format and are loaded only at evaluation time: the evaluation routine populates a separate buffer for context-conditioned adaptation on out-of-support tasks and never feeds out-of-support transitions into any training gradient.

\paragraph{Context sampling and preprocessing.}
At each meta-training step, we sample a meta-batch of in-distribution task indices and, for each task, draw an i.i.d.\ batch of transitions from the corresponding per-task offline buffer; the encoder $q_\phi(\zz \mid c)$ consumes this batch as the context $c$ and infers the per-task latent $\zz$. Context batches are i.i.d.\ rather than sequential, as in the PEARL family of methods~\citep{rakelly2019efficient}. Observations are passed unnormalized, and rewards are multiplied by the standard reward scale of $5.0$, shared across all baselines.

\section{Implementation Details}
\label{app:impl}

This appendix collects the implementation details needed to reproduce PLATO. \Cref{tab:hp_shared} lists the hyperparameters shared across all environments, and \Cref{tab:hp_env} lists the four hyperparameters that take per-environment values.

\paragraph{Architectures.}
The context encoder $q_\phi(\zz \mid c)$ is a permutation-invariant set encoder following \citep{rakelly2019efficient}: each transition is mapped through a $3$-layer MLP with hidden size $200$ and ReLU activations, the per-transition embeddings are mean-pooled, and a final tanh-squashed projection produces the latent $\zz \in \R^{20}$. The encoder input depends on the task family: the reward-shift families use $(s, a, r)$, since their tasks share dynamics and $s'$ carries no task information, and the dynamics-shift families (Hopper and Walker, Mass and Friction) use $(s, a, r, s')$; PLATO and all baselines follow the same rule. The decoder ensemble has $M = 10$ bootstrapped members, each a $3$-layer MLP with hidden size $256$ that maps $(s, a, \zz)$ to the predicted normalized state delta $\Delta\hat s$ and reward. The twin critics $Q_{\omega_1}, Q_{\omega_2}$ and the value network $V$ are $3$-layer MLPs with hidden size $256$. The policy $\pi_\theta(a \mid s, \zz)$ is a Gaussian with a $3$-layer MLP backbone of hidden size $256$, parameterizing a state-dependent mean and log-std with a $\tanh$ squashing transform. The CLUB network is a $3$-layer MLP of hidden size $200$ that outputs the parameters of a Gaussian over $\zz$ given the transition features $x$ defined below.

\paragraph{Optimization.}
All networks are trained with Adam~\citep{kingma2015adam} at learning rate $3 \times 10^{-4}$, with two exceptions: the BRAC divergence critic $c_\omega$ introduced below uses $1 \times 10^{-4}$, and on Ant-Goal the context encoder and the CLUB network use $6 \times 10^{-4}$. Each training step samples a meta-batch of $B_T = 16$ tasks; for each task, we draw a context batch of $64$ transitions and a separate transition batch of $256$ transitions for the actor--critic update. The decoder ensemble uses a bootstrap fraction of $\rho = 0.8$. Target networks are soft-updated with rate $\tau = 0.005$. Rewards are scaled by $5.0$ before entering the critic, matching the convention used by all baselines.

\paragraph{Backbone objectives.}
The encoder is trained with the FOCAL distance-metric loss and the CLUB upper bound, with weights inherited from CSRO~\citep{gao2023csro}, so that the encoder backbone is identical across all PLATO and CSRO runs. For latents $\zz_i, \zz_j$ inferred from two contexts in a meta-batch, let $\delta_{ij} := \tfrac{1}{d}\|\zz_i - \zz_j\|^2$, and let $\mathcal{P}$ and $\mathcal{N}$ be the sets of pairs drawn from the same task and from different tasks, respectively. The FOCAL loss~\citep{li2020focal} pulls latents of the same task together and pushes latents of different tasks apart,
\[
  \cL_\text{FOCAL} \;=\; \frac{1}{|\mathcal{P}|}\sum_{(i,j)\in\mathcal{P}} \sqrt{\delta_{ij} + \epsilon} \;+\; \frac{1}{|\mathcal{N}|}\sum_{(i,j)\in\mathcal{N}} \frac{1}{\delta_{ij} + \epsilon'}.
\]
CLUB~\citep{cheng2020club} fits a variational network $q_\varphi(\zz \mid x) = \cN\bigl(\mu_\varphi(x), \mathrm{diag}\,\sigma^2_\varphi(x)\bigr)$ to the encoder's latents by maximum likelihood, where $x$ is $(s, a)$, or $(s, a, r)$ when $s'$ is part of the context, and the encoder minimizes the resulting upper bound on $\cI(\zz; x)$,
\[
  \hat\cI_\text{CLUB} \;=\; \frac{1}{N}\sum_{i} \log q_\varphi(\zz_i \mid x_i) \;-\; \frac{1}{N^2}\sum_{i,j} \log q_\varphi(\zz_i \mid x_j),
\]
which removes from $\zz$ the information that $(s, a)$ alone carries about the behavior policy. The actor--critic follows BRAC~\citep{wu2019behavior} with the dual form of the KL divergence: with a critic $c_\omega(s, a, \zz)$ and $g(x) = \log(1 + e^{x})$,
\[
  D_\text{BRAC}(s, a', a; \zz) \;=\; \log g\bigl(c_\omega(s, a', \zz)\bigr) - g\bigl(c_\omega(s, a, \zz)\bigr) + 1
\]
for a policy action $a' \sim \pi_\theta$ and a dataset action $a$. The critic $c_\omega$ maximizes this estimate with a gradient penalty, the actor minimizes $\alpha_\text{BRAC} D_\text{BRAC}$ as part of \Cref{eq:actor_real}, and $\alpha_\text{BRAC}$ is adapted by dual gradient descent toward a small target divergence.

\paragraph{PLATO-specific hyperparameters.}
The edge-emphasis exponent in \Cref{eq:lore} is fixed at $p = 2$ across all environments. Four hyperparameters take per-environment values: the perturbation magnitude $\alpha$, the pessimism strength $\lambda_\text{pess}$, the rollout horizon $H$, and the warm-up $T_\text{warm}$; \Cref{tab:hp_env} lists their values. The synthetic-loss weight $\beta_\text{synth}$ in \Cref{eq:actor_synth} is set jointly with $H$: a smaller $H$ uses a smaller $\beta_\text{synth}$, which keeps the synthetic gradient on a scale comparable to that of the real-data gradient. The entropy temperature $\eta$ of \Cref{eq:actor_real,eq:actor_synth}, which also enters the soft value target, is fixed at $\eta = 1$ in all environments except Hopper-Mass and Walker-Mass. On these two, it grows with the decoder disagreement as $\eta = \eta_0 (1 + \kappa\,\tilde u)$, where $\tilde u$ is the disagreement $u$ of \Cref{eq:u} standardized by exponential moving averages of its batch mean and variance (decay $0.99$), rectified, clipped at $5$, and held at zero for the first $5$K steps; $(\eta_0, \kappa) = (1.0, 1.8)$ on Hopper-Mass and $(1.5, 0.7)$ on Walker-Mass. Sensitivity around the chosen $\alpha$, $\lambda_\text{pess}$, and $T_\text{warm}$ is reported in \Cref{app:sens}.

\paragraph{Training budget and reporting.}
Each run trains for $200{,}000$ policy-gradient updates ($100$ outer iterations of $2{,}000$ inner steps), and all eight environments use this identical budget. Final-iteration metrics are computed by averaging $5$ evaluation rollouts per task and then averaging across $5$ random seeds. Out-of-support transitions are loaded only at evaluation time and never enter any training gradient.

\begin{table}[t]
\centering\small
\begin{minipage}[t]{0.46\textwidth}
\centering
\caption{\textbf{Shared hyperparameters across all environments.}}
\label{tab:hp_shared}
\begin{tabular}{lc}
\toprule
Hyperparameter & Value \\
\midrule
Latent dimension $d$                 & $20$ \\
Encoder MLP                          & $3 \times 200$ \\
Decoder MLP (per ensemble member)    & $3 \times 256$ \\
Critic / value / policy MLP          & $3 \times 256$ \\
Decoder ensemble size $M$            & $10$ \\
Bootstrap fraction $\rho$            & $0.8$ \\
Edge-emphasis exponent $p$           & $2$ \\
Discount factor $\gamma$             & $0.99$ \\
Soft-update rate $\tau$              & $0.005$ \\
Meta-batch size $B_T$                & $16$ \\
Transition batch size                & $256$ \\
Context batch size                   & $64$ \\
Reward scale                         & $5.0$ \\
Learning rate (default)              & $3 \times 10^{-4}$ \\
Optimizer                            & Adam \\
Total training steps                 & $200{,}000$ \\
\bottomrule
\end{tabular}
\end{minipage}
\hfill
\begin{minipage}[t]{0.5\textwidth}
\centering
\caption{\textbf{Per-environment hyperparameters.}}
\label{tab:hp_env}
\begin{tabular}{lcccc}
\toprule
Environment & $\alpha$ & $\lambda_\text{pess}$ & $H$ & $T_\text{warm}$ \\
\midrule
Ant-Dir          & $1.5$  & $4.0$  & $3$ & $10$\,K \\
Ant-Goal         & $1.5$  & $4.0$  & $3$ & $10$\,K \\
Cheetah-Vel      & $1.0$  & $2.5$  & $3$ & $10$\,K \\
Humanoid-Dir     & $1.5$  & $2.5$  & $3$ & $10$\,K \\
\midrule
Hopper-Mass      & $0.1$  & $5.0$  & $2$ & $10$\,K \\
Hopper-Friction  & $0.1$  & $5.0$  & $2$ & $25$\,K \\
Walker-Mass      & $0.7$  & $1.0$  & $5$ & $10$\,K \\
Walker-Friction  & $0.3$  & $2.5$  & $3$ & $10$\,K \\
\bottomrule
\end{tabular}
\end{minipage}
\end{table}

\section{Computational Cost}
\label{app:compute}

\paragraph{Hardware.}
All training runs were launched on a single shared workstation with $6$ NVIDIA RTX $3090$ GPUs ($24$\,GB VRAM each), two AMD EPYC $7452$ processors ($64$ physical cores and $128$ threads in total), and $503$\,GB of system RAM. Each run uses a single GPU, and up to six runs are time-multiplexed concurrently across the workstation.

\paragraph{Per-run wall-clock time.}
Wall-clock time per training run varies with environment complexity and with contention on the shared GPUs. After trimming the slowest $10$\% of seeds per cell to remove obvious contention outliers, the median PLATO times across the $5$ seeds were approximately:
\begin{itemize}\itemsep1pt
  \item \emph{Reward-shift:} Cheetah-Vel $\sim$$7$\,h, Ant-Goal $\sim$$18$\,h, Ant-Dir $\sim$$28$\,h, Humanoid-Dir $\sim$$30$\,h.
  \item \emph{Dynamics-shift:} Hopper-Mass, Walker-Mass, and Walker-Friction $\sim$$22$\,h; Hopper-Friction $\sim$$24$\,h.
\end{itemize}
The total time per environment is dominated by the cost of environment steps: Cheetah-Vel uses half as many simulation steps per outer iteration as the other environments and is therefore the fastest, while the $376$-dimensional state of Humanoid-Dir makes its critic forward and backward passes the slowest.

\paragraph{Overhead vs.~baselines.}
Compared to CSRO~\citep{gao2023csro}, the strongest baseline in our experiments and the closest in design (PLATO inherits the CSRO encoder and the BRAC actor--critic), PLATO incurs roughly a $5$--$15$\% wall-clock overhead from the $H$-step decoder-ensemble rollout and the per-step uncertainty computation. Precise per-environment numbers are masked by contention on the shared GPUs but cluster around $+10$\% after outlier removal. The overhead is constant in the number of training tasks and does not grow with the out-of-support evaluation set, since out-of-support transitions are never used during training.

\paragraph{Memory.}
Peak GPU memory during training is dominated by the forward and backward passes of the decoder ensemble ($M = 10$ members at meta-batch size $B_T = 16$) and stays well under $5$\,GB, so six concurrent training runs fit comfortably on the six $24$\,GB GPUs of the workstation. Meta-test inference uses only the encoder and the policy and consumes well under $1$\,GB.

\section{Sensitivity Analyses}
\label{app:sens}

The hyperparameters reported in \Cref{tab:hp_env} are the per-environment values used for the main result tables (\Cref{tab:id,tab:ood}). To probe PLATO's robustness around these choices, we run a separate sensitivity study in which each of the three main hyperparameters (the perturbation magnitude $\alpha$ of \Cref{eq:lore}, the pessimism strength $\lambda_\text{pess}$ of \Cref{eq:weight}, and the warm-up $T_\text{warm}$ that delays the onset of imagined rollouts) is swept over three values, applied uniformly across environments. We report out-of-support learning curves on four representative environments, two from each shift family: Ant-Goal and Cheetah-Vel (reward shift), and Hopper-Mass and Walker-Friction (dynamics shift). In each sweep, the other two hyperparameters are held fixed at their middle swept value. Curves show the mean over 5 seeds.

\Cref{fig:sensitivity} reports the result. Overall, PLATO is broadly insensitive to all three hyperparameters within reasonable ranges: final returns mostly cluster within one seed standard deviation across the swept values. A few patterns are nevertheless worth noting.

\begin{figure}[!htb]
\centering
\includegraphics[width=\linewidth]{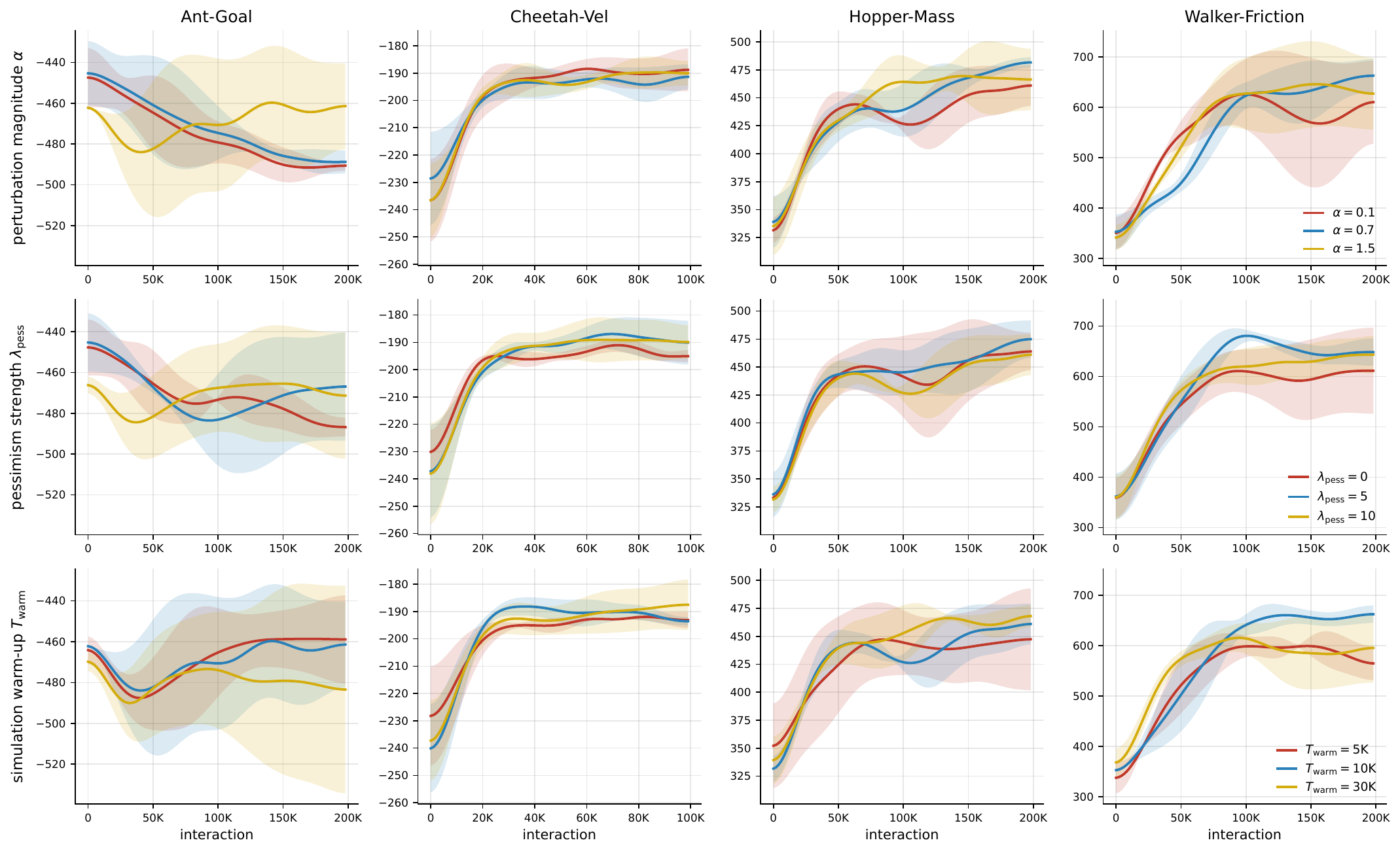}
\caption{\textbf{Out-of-support sensitivity analysis.} Each row sweeps one hyperparameter, with the other two held at their middle swept value. Top: perturbation magnitude $\alpha \in \{0.1, 0.7, 1.5\}$. Middle: pessimism strength $\lambda_\text{pess} \in \{0, 5, 10\}$. Bottom: warm-up $T_\text{warm} \in \{5\text{K}, 10\text{K}, 30\text{K}\}$. Columns: Ant-Goal and Cheetah-Vel (reward-shift), Hopper-Mass and Walker-Friction (dynamics-shift). Curves show mean out-of-support return over 5 seeds.}
\label{fig:sensitivity}
\end{figure}

\paragraph{Perturbation magnitude $\alpha$ (top row).}
The reach of the perturbation governs both the off-support distance attained by $\zz^\text{aug}$ (\Cref{prop:outward}) and the strength of the resulting disagreement signal (\Cref{rem:lore_bound}). In most environments, the three swept values produce comparable returns, but Ant-Goal is a clear exception: $\alpha=1.5$ yields a substantially higher final return than $\alpha=0.7$ or $\alpha=0.1$. We attribute this to the geometry of the Ant-Goal task family. Held-out goals lie at a strictly larger radius than training goals \emph{and} in a different angular sector, so the out-of-support latents are displaced from the training centroid in two directions simultaneously. A larger $\alpha$ pushes $\zz^\text{aug}$ far enough to overlap with this doubly displaced out-of-support region, whereas smaller values land in the interpolation region between training tasks and miss the held-out region entirely. The other three environments hold out tasks along a single axis (forward velocity for Cheetah-Vel and a one-dimensional physics parameter for Hopper-Mass and Walker-Friction), so the perturbation reach matters less. The pattern is consistent with \Cref{prop:gengap}: the out-of-support generalization gap is bounded by the perturbation reach, and the bound becomes informative only when the reach is large enough to cover the held-out region.

\paragraph{Pessimism strength $\lambda_\text{pess}$ (middle row).}
The middle row sweeps the pessimism strength of \Cref{eq:weight}. Disabling pessimism ($\lambda_\text{pess}=0$) and over-emphasizing it ($\lambda_\text{pess}=10$) both perform slightly worse than the moderate setting $\lambda_\text{pess}=5$, which gives the best returns most consistently across the four environments. The two failure modes are intuitive. At $\lambda_\text{pess}=0$, the policy treats every imagined transition as fully reliable and absorbs decoder errors at $\zz^\text{aug}$, where the ensemble was never trained. At $\lambda_\text{pess}=10$, the weight $w_h$ shrinks too aggressively even on transitions with moderate disagreement, suppressing the out-of-support training signal. The intermediate value strikes the balance the algorithm is designed to exploit.

\paragraph{Warm-up $T_\text{warm}$ (bottom row).}
$T_\text{warm}$ delays imagined rollouts until the decoder ensemble has stabilized, and its effect is mild. On Cheetah-Vel and Hopper-Mass the three values largely overlap, with the longest warm-up marginally ahead at the end of training. On Ant-Goal the longest warm-up ($30$K) ends lowest and with the largest spread across seeds, while on Walker-Friction the middle value ($10$K) is clearly best. No setting dominates across environments, and the $10$K warm-up that these four environments use in the main experiments is never far from the best value in the sweep.

\paragraph{Ensemble size $M$.}
PLATO uses $M = 10$ decoders (\Cref{sec:backbone}). We study the ensemble size in two ways. First, without retraining, we evaluate the trained $M = 10$ ensembles of all eight environments using random subsets of $M' \in \{2, 3, 5, 7, 10\}$ members: the separation between training and out-of-support tasks under $u$, measured by the area under the receiver operating characteristic curve (AUROC) of the per-task disagreement, changes by at most $0.02$ between $M' = 10$ and $M' = 2$ in every environment. A small ensemble therefore preserves the disagreement signal on average; what it loses is the reliability of a single draw, as the AUROC varies more across random subsets when $M'$ is small. Second, we retrain PLATO with $M \in \{3, 5\}$ on three environments (\Cref{tab:ens_size}). Out-of-support return stays within $6\%$ of the $M = 10$ value on Cheetah-Vel and Hopper-Mass, and $M = 5$ matches $M = 10$ on Walker-Friction; the main effect of a small ensemble is a larger spread across seeds, in line with a less reliable disagreement estimate. Since meta-testing uses no ensemble (\Cref{app:algorithm}) and $M = 10$ costs about $10\%$ additional training time (\Cref{app:compute}), we keep $M = 10$.

\begin{table}[!htb]
\caption{\textbf{Out-of-support return for different ensemble sizes $M$.} $M = 10$ is the main configuration (\Cref{tab:ood}, $5$ seeds); $M \in \{3, 5\}$ are averaged over $3$ seeds.}
\label{tab:ens_size}
\centering\small
\begin{tabular}{lccc}
\toprule
$M$ & Cheetah-Vel & Hopper-Mass & Walker-Friction \\
\midrule
$3$  & $-186 \pm 12$ & $479 \pm 36$ & $564 \pm 145$ \\
$5$  & $-187 \pm 2$  & $478 \pm 8$  & $655 \pm 53$ \\
$10$ & $-179 \pm 11$ & $508 \pm 26$ & $649 \pm 27$ \\
\bottomrule
\end{tabular}
\end{table}

\paragraph{Summary.}
PLATO's three hyperparameters are mostly soft choices: across the four environments and the swept ranges, final returns rarely differ by more than one seed standard deviation. Two patterns deserve attention. First, $\alpha$ matters when the held-out task region is distant in several directions at once, as in Ant-Goal, where a larger reach pays off. Second, $\lambda_\text{pess}$ at extreme values ($0$ or $10$) underperforms the moderate setting. $T_\text{warm}$ shows mild environment-dependent trade-offs, with no clearly dominant setting. The ensemble size behaves similarly: smaller ensembles leave the disagreement signal and the out-of-support return largely unchanged and mainly increase variability across seeds.

\section{Learning Curves}
\label{app:curves}

\Cref{fig:curves_id,fig:curves_ood} report the learning curves of PLATO and the five baselines on the eight MuJoCo OMRL benchmarks, for in-distribution and out-of-support evaluation, respectively. Curves are averaged across 5 seeds, and shaded regions indicate one standard deviation.

\begin{figure}[!htb]
\centering
\includegraphics[width=\linewidth]{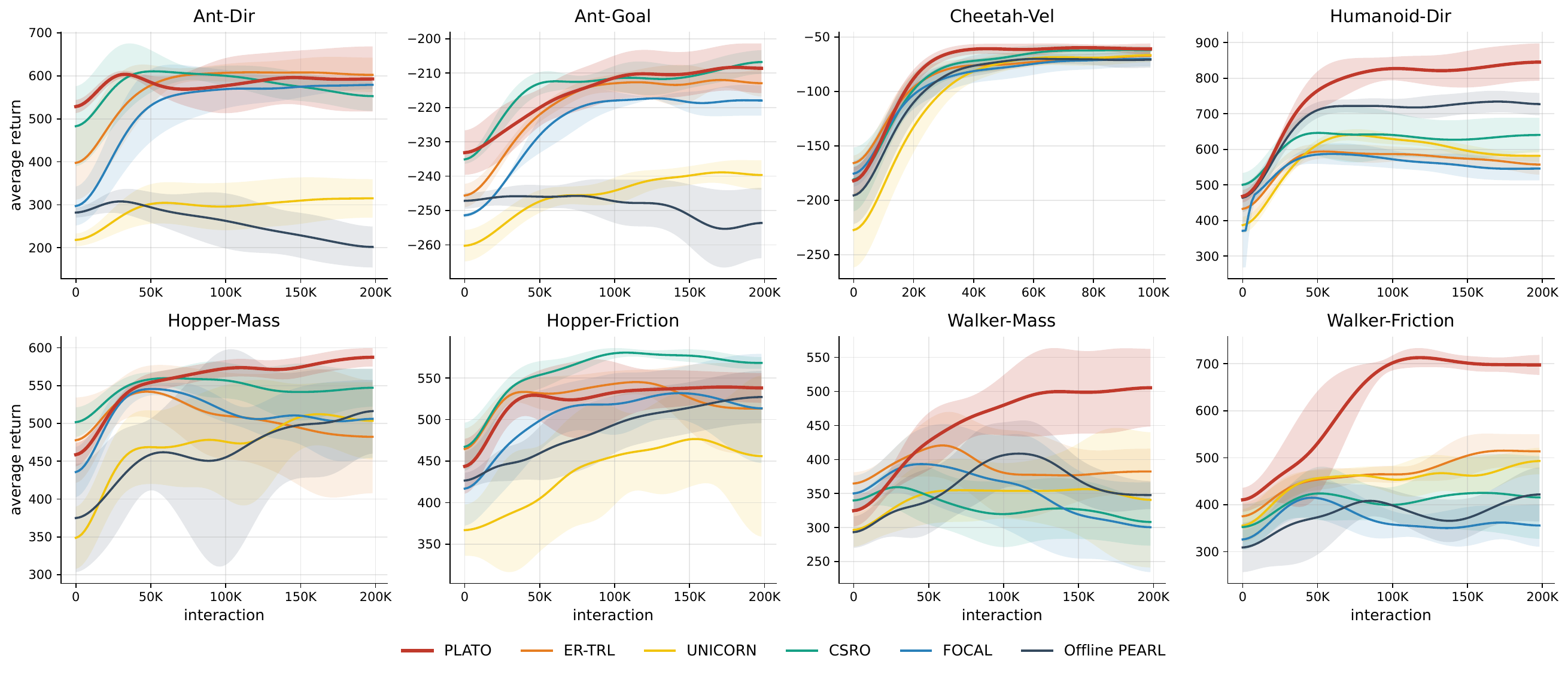}
\caption{\textbf{Learning curves on the eight MuJoCo OMRL benchmarks, in-distribution.} Each panel shows normalized return over training, averaged across 5 seeds; shaded regions indicate one standard deviation.}
\label{fig:curves_id}
\end{figure}

\begin{figure}[!htb]
\centering
\includegraphics[width=\linewidth]{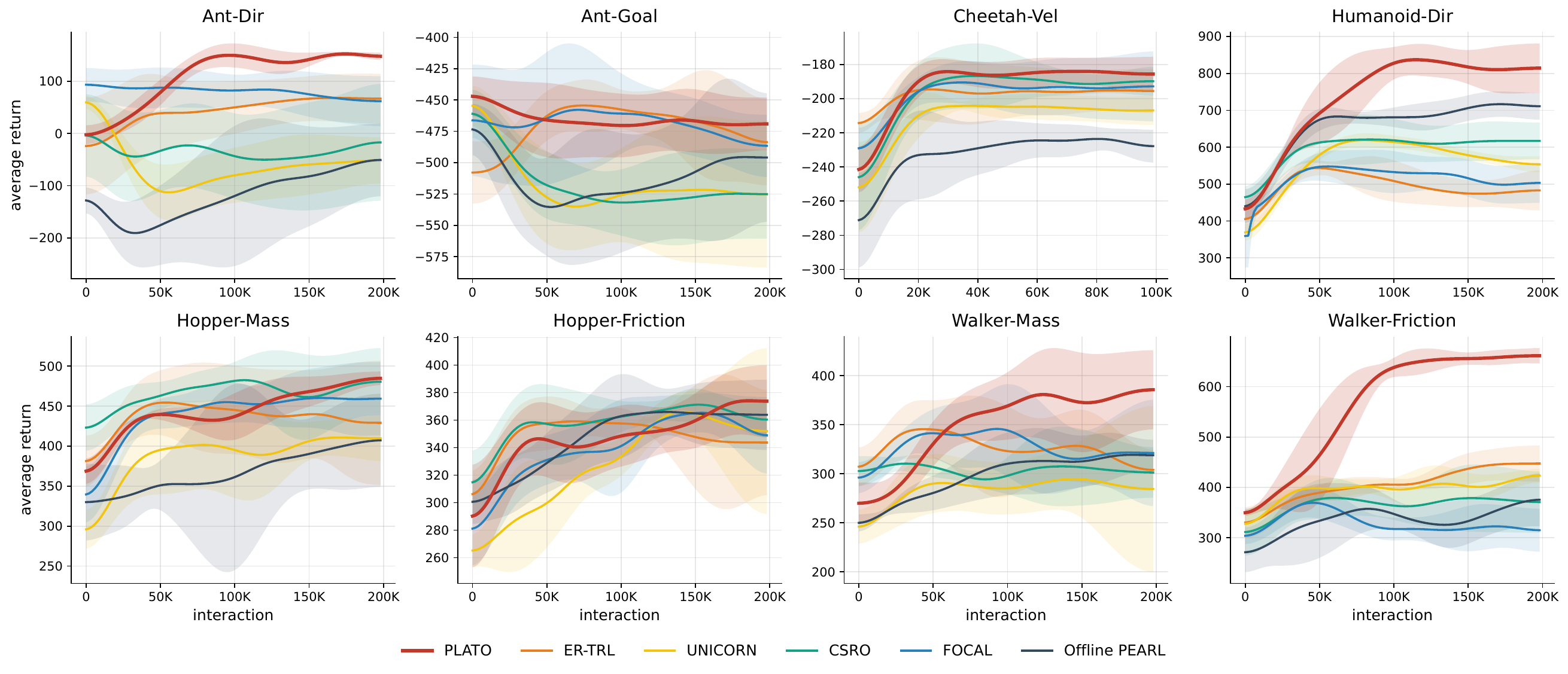}
\caption{\textbf{Learning curves on the eight MuJoCo OMRL benchmarks, out-of-support.} Same format as \Cref{fig:curves_id}.}
\label{fig:curves_ood}
\end{figure}

\section{Additional Domain: MetaWorld Reach}
\label{app:metaworld}

To evaluate PLATO beyond locomotion, we add MetaWorld ML1 Reach~\citep{yu2019metaworld}, a state-based manipulation task in which a robot arm must bring its end effector to a goal position in 3D. Its $20$ training tasks and $10$ out-of-support tasks come from a grid over the goal box with the periphery held out, so every held-out goal lies outside the region covered by the training goals (\Cref{app:env-manip}). Offline data are collected per task with SAC as in \Cref{app:data}, and all six methods share the datasets, training budget, and evaluation protocol of the main experiments. We report the MetaWorld success rate (the fraction of evaluation episodes that reach the goal) and the average return; the numbers below are the mean $\pm$ standard deviation over $3$ seeds at the last iteration, as in \Cref{tab:ood,tab:id}.

\begin{figure}[!htb]
\centering
\includegraphics[width=\linewidth]{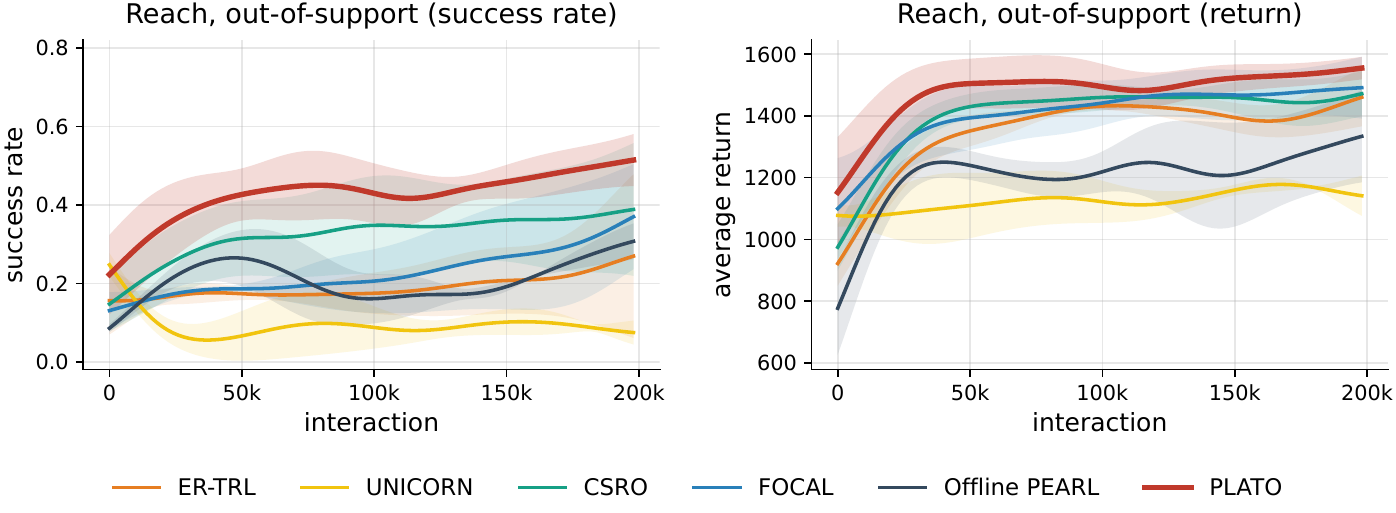}
\caption{\textbf{MetaWorld Reach, out-of-support tasks.} Success rate (left) and average return (right) over training, averaged over $3$ seeds; shaded regions indicate one standard deviation.}
\label{fig:metaworld}
\end{figure}

\Cref{fig:metaworld} shows the learning curves on the out-of-support tasks. PLATO attains the highest out-of-support success rate ($0.53 \pm 0.09$, against $0.43 \pm 0.13$ for the strongest baseline, FOCAL) and the highest out-of-support return ($1575 \pm 39$, against $1506 \pm 78$ for CSRO), with the smallest seed spread in return among all methods, and it leads for most of training rather than only at the final iteration. As on MuJoCo, the out-of-support gain comes while PLATO remains competitive in distribution: on the training tasks, PLATO matches the best return ($1692 \pm 28$, against $1693 \pm 90$ for FOCAL) with the smallest seed spread, and its success rate ($0.60 \pm 0.11$) is close to those of the strongest baselines ($0.65 \pm 0.07$ for FOCAL and $0.67 \pm 0.24$ for Offline PEARL). When we extrapolate further, to the corners of the goal box at about $1.9\times$ the training radius, the success rate of every method falls to at most $0.10$, although PLATO still attains the highest return there ($1108$, against $1083$ for FOCAL). This is consistent with the view that PLATO extends the range over which extrapolation remains reliable rather than removing its limit.

\section{Comparison with MerPO}
\label{app:merpo}

MerPO~\citep{lin2022merpo} is the model-based offline meta-RL method closest to PLATO: it learns a model of the task family, mixes imagined transitions from this model into the Bellman target of a conservative policy optimization, and adapts the policy to each new task with gradient steps on that task's offline data. Its within-task policy optimization generalizes COMBO~\citep{yu2021combo}, which combines model-based policy optimization with the conservative policy evaluation of CQL~\citep{kumar2020conservative}, to the multi-task setting. As no official implementation is available, we ported the authors' code from the OpenReview supplementary material into our codebase. Its data loader reads the FOCAL format, so MerPO trains on exactly the datasets and task splits used by PLATO. We kept its two-stage training pipeline and the authors' per-environment hyperparameters, changing only what prevented the released code from running on non-contiguous task splits. Because MerPO takes gradient steps on test-task data at evaluation time, whereas PLATO infers each task from a context batch without any gradient step, we treat MerPO as a test-time-adaptation baseline.

\begin{table}[!htb]
\caption{\textbf{Comparison with MerPO.} Returns on the in-distribution and out-of-support splits, mean $\pm$ standard deviation over 6 seeds at the last epoch for MerPO; the PLATO values are those of \Cref{tab:ood,tab:id}. Boldface marks the higher average return in each pair.}
\label{tab:merpo}
\centering\small\setlength{\tabcolsep}{8pt}
\begin{tabular}{lcccc}
\toprule
 & \multicolumn{2}{c}{In-distribution} & \multicolumn{2}{c}{Out-of-support} \\
\cmidrule(lr){2-3}\cmidrule(lr){4-5}
Environment & MerPO & PLATO & MerPO & PLATO \\
\midrule
Ant-Dir          & \textbf{678$\pm$34} & 617$\pm$29 & 70$\pm$18 & \textbf{139$\pm$25} \\
Ant-Goal         & -225$\pm$12 & \textbf{-215$\pm$10} & -481$\pm$20 & \textbf{-472$\pm$20} \\
Humanoid-Dir     & 652$\pm$51 & \textbf{858$\pm$43} & 577$\pm$68 & \textbf{806$\pm$97} \\
\midrule
Hopper-Mass      & \textbf{602$\pm$19} & 584$\pm$12 & \textbf{516$\pm$24} & 508$\pm$26 \\
Hopper-Friction  & \textbf{576$\pm$33} & 526$\pm$42 & 329$\pm$27 & \textbf{369$\pm$29} \\
Walker-Mass      & 463$\pm$41 & \textbf{507$\pm$53} & 328$\pm$47 & \textbf{401$\pm$54} \\
Walker-Friction  & 494$\pm$22 & \textbf{682$\pm$18} & 464$\pm$31 & \textbf{649$\pm$27} \\
\bottomrule
\end{tabular}
\end{table}

\Cref{tab:merpo} compares the two methods on seven of the eight MuJoCo environments. In distribution, MerPO attains the higher return on Ant-Dir and the two Hopper environments, consistent with its test-time adaptation, and PLATO on the other four. Out of support, PLATO is ahead by more than two standard deviations on Ant-Dir, Humanoid-Dir, and Walker-Friction and by more than one on Hopper-Friction and Walker-Mass, and the two methods are within one standard deviation of each other on Ant-Goal and Hopper-Mass. PLATO's out-of-support advantage therefore holds against a baseline that adapts to each test task with gradient steps, although PLATO itself takes no gradient step at test time. We omit Cheetah-Vel, on which MerPO did not train stably under the authors' hyperparameters or three further configurations we tried.

\section{Limitations}
\label{app:limitations}

\paragraph{Where the manifold assumption holds.}
\Cref{asm:manifold} is a sufficient condition for the outward direction to point toward held-out tasks, and it is not vacuous. In an annulus configuration, held-out tasks lie inside an outer training envelope rather than beyond it, so the centroid-outward direction is miscalibrated. We constructed this case on Ant-Goal by training on the outer goal radii and holding out the inner one. As expected, the direction then steers synthetic coverage away from the held-out goals: among the variants we compared, PLATO attains the highest return on the outer training goals, while on the inner held-out goals a random perturbation direction performs better and PLATO stays within the seed spread of its unperturbed CSRO backbone. Mild violations, such as Walker-Mass in \Cref{app:env-pca}, do not cost PLATO its advantage; a full inversion of the task geometry does. The disagreement signal does not flag such an inversion from training data alone, since probing $u$ along latent rays shows it rising outward but staying flat inward. An annulus-type geometry must therefore be ruled out by domain knowledge, and PLATO's stated guarantees do not transfer to it.

We argue, however, that the outward configuration is the common case rather than a special one. Out-of-support splits in the meta-RL literature are constructed by holding out the periphery of a parameter range, which is the natural way for a practitioner to test extrapolation. The same geometry recurs in real applications: a household robot trained on the mass range of typical kitchenware will encounter heavier outliers, not lighter-than-trained ones; a delivery robot tuned for nominal floor friction will face unusually slippery or unusually rough surfaces, not interpolations between them; and a manipulation policy trained on a band of object sizes will need to extrapolate to larger or smaller objects. In each case, the held-out tasks sit on the outer boundary of the training distribution. PLATO targets this regime directly, and we view the outward assumption as a reasonable description of how out-of-support generalization is encountered in practice rather than a contrived simplification.

\paragraph{Other limitations.}
The centroid $c$ and the spread $\sigma_d$ are estimated from a finite mini-batch and inherit its estimation noise. The reach $\alpha$ and the horizon $H$ trade coverage against rollout drift: small values make the simulated out-of-support tasks indistinguishable from training tasks, and large values make the rollout diverge from the true dynamics. \Cref{thm:decoupling} formalizes the bias of training the critic on imagined data, but the trade-off is environment-dependent, and we treat the inclusion of imagined data in the critic as an empirical knob. \Cref{thm:nonvanish} is a conditional guarantee: we verify its conclusion empirically (\Cref{app:proof_nonvanish}) rather than derive its margin assumption. Extrapolation also has a magnitude limit: far beyond the training range, as at the corners of the MetaWorld goal box, every method including PLATO fails (\Cref{app:metaworld}), so PLATO extends the range over which extrapolation remains reliable rather than removing its limit. Our main benchmarks are MuJoCo locomotion tasks, complemented by one manipulation task (\Cref{app:metaworld}); vision-based domains and transfer across qualitatively different task families remain for future work.

\section{Broader Impact}
\label{app:broader_impact}

PLATO improves the robustness of offline meta-RL policies on tasks beyond the range of the training tasks. The motivating applications, including robotics, healthcare, and autonomous systems, are domains in which deployed policies routinely encounter situations beyond the training distribution and in which unsafe online interaction is a primary reason to choose offline meta-RL in the first place. By reducing the gap between training-time exposure and deployment-time conditions, methods in this line aim to make such systems more reliable.

The same property carries dual-use risks common to any work that improves policy generalization: a policy that adapts more confidently to unfamiliar situations can also fail more confidently when the situation lies beyond what the method's assumptions support. PLATO is explicit about this regime: the analysis in \Cref{sec:analysis} bounds the out-of-support generalization gap by the perturbation reach, not by the ambient task space. Practitioners deploying offline meta-RL agents in safety-critical settings should therefore treat the off-support guarantees as local rather than universal and combine them with appropriate runtime safeguards.

\end{document}